\documentclass{article}

\usepackage{microtype}
\usepackage{graphicx}
\usepackage{caption}
\usepackage{subcaption}
\usepackage{booktabs} 

\PassOptionsToPackage{hyphens}{url}\usepackage{hyperref}

\usepackage[preprint]{icml2026}

\usepackage{amsmath}
\usepackage{amssymb}
\usepackage{mathtools}
\usepackage{amsthm}
\usepackage{bm}

\usepackage{xparse}
\usepackage{xcolor}
\usepackage{enumitem}
\usepackage{multirow}
\usepackage{tabularx}

\usepackage[capitalize,noabbrev]{cleveref}

\usepackage[normalem]{ulem}
\usepackage{pifont}
\usepackage{ant}

\theoremstyle{plain}
\newtheorem{theorem}{Theorem}[section]
\newtheorem{proposition}[theorem]{Proposition}
\newtheorem{lemma}[theorem]{Lemma}
\newtheorem{corollary}[theorem]{Corollary}
\theoremstyle{definition}
\newtheorem{definition}[theorem]{Definition}

\theoremstyle{remark}

\usepackage[disable, textsize=tiny]{todonotes}

\newcommand{\xmark}{\ding{55}}

\newcommand{\ddt}{\frac{\mathrm{d}}{\mathrm{d}t}}
\newcommand{\R}{\mathbb{R}}

\icmltitlerunning{Generator-based Graph Generation via Heat Diffusion}

\begin{document}

\twocolumn[
  \icmltitle{Generator-based Graph Generation via Heat Diffusion }


  \icmlsetsymbol{equal}{*}

  \begin{icmlauthorlist}
    \icmlauthor{Anthony Stephenson}{lanc,bris}
    \icmlauthor{Ian Gallagher}{melbourne}
    \icmlauthor{Christopher Nemeth}{lanc}
  \end{icmlauthorlist}

  \icmlaffiliation{bris}{Department of Mathematics, University of Bristol, Bristol BS8 1UG, UK.}
  \icmlaffiliation{melbourne}{School of Mathematics and Statistics,
       University of Melbourne,
       Parkville, VIC, 3010, Australia.}
  \icmlaffiliation{lanc}{MARS: Mathematics for AI in Real-world Systems, School of Mathematical Sciences
       Lancaster University
       Lancaster, LA1 4YF, UK}

  \icmlcorrespondingauthor{Anthony Stephenson}{ant.stephenson@bristol.ac.uk}
  \icmlcorrespondingauthor{Ian Gallagher}{ian.gallagher@unimelb.edu.au}
  \icmlcorrespondingauthor{Christopher Nemeth}{c.nemeth@lancaster.ac.uk}

  \icmlkeywords{Machine Learning, ICML}

  \vskip 0.3in
]



\printAffiliationsAndNotice{}  

\begin{abstract}
Graph generative modelling has become an essential task due to the wide range of applications in chemistry, biology, social networks, and knowledge representation. In this work, we propose a novel framework for generating graphs by adapting the Generator Matching \citep{holderrieth2024generator} paradigm to graph‐structured data. We leverage the graph Laplacian and its associated heat kernel to define a continuous‐time diffusion on each graph. The Laplacian serves as the infinitesimal generator of this diffusion, and its heat kernel provides a family of conditional perturbations of the initial graph. A neural network is trained to match this generator by minimising a Bregman divergence between the true generator and a learnable surrogate. Once trained, the surrogate generator is used to simulate a time‐reversed diffusion process to sample new graph structures. 
Our framework unifies and generalises existing diffusion‐based graph generative models, injecting domain‐specific inductive bias via the Laplacian, while retaining the flexibility of neural approximators. Experimental studies demonstrate that our approach captures structural properties of real and synthetic graphs effectively.
\end{abstract}

\section{Introduction}
\label{sec:intro}

Graph generative modelling has emerged as a critical research area due to the broad applicability of graphs in representing complex systems such as molecular structures \citep{Sylvester1878ChemistryAA, de2018molgan}, social networks \citep{nettleton2013data}, and knowledge graphs \citep{schneider1973course, ji2021survey}. The goal in graph generation is to learn the underlying distribution of a collection of observed graphs and generate new samples that exhibit similar structural properties. Applications range from de novo molecule design in chemistry \citep{brown2019guacamol} to synthetic network construction for privacy‐preserving data sharing \citep{fu2023privacy}. However, \emph{structured} data like graphs pose unique challenges for generative models, owing to their discrete combinatorial nature and the permutation symmetries among node labels. 

Early approaches to graph generation include graph-structured variational autoencoders \citep{kipf2016variational}, autoregressive models \citep{you2018graphrnn}, and generative adversarial networks \citep{wang2018graphgan}. While these methods achieved partial success, they often struggled to capture higher-order dependencies or to enforce fundamental invariances such as permutation symmetry. More recently, score-based generative models—also known as diffusion models—have demonstrated remarkable performance in continuous domains such as images by progressively transforming noise into data via a forward diffusion and learned reverse process \citep{song2020score}. This success has spurred increasing interest in adapting diffusion models to graph-structured data, giving rise to \emph{graph diffusion models}. Existing approaches typically operate on adjacency matrices or node and edge features, and have shown promising results in tasks such as molecule generation and network completion \citep{vignac2022digress,yang2023score}. Notably, some recent methods formulate permutation-invariant stochastic differential equations that jointly diffuse node and edge attributes \citep{jo2022score}, while others apply diffusion in a continuous latent space obtained via graph autoencoders. Despite these advances, defining principled and effective diffusion processes \emph{directly} on discrete graph structures remains challenging, motivating the development of more general, generator-centric frameworks.

\begin{figure*}[t]
    \centering
    \begin{subfigure}[h]{0.65\textwidth}
    \includegraphics[width=\linewidth]{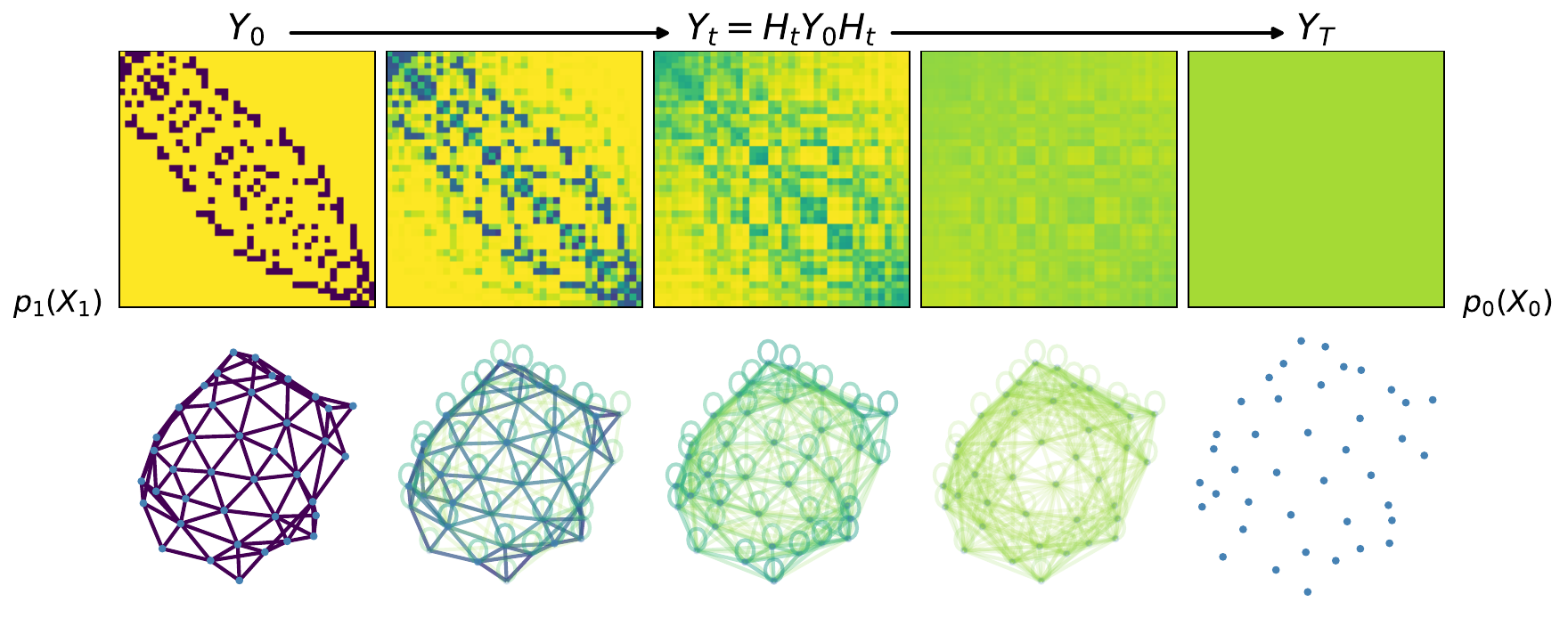}
    \caption{}\label{fig:forward_process}
    \end{subfigure}
    \hfill
    \begin{subfigure}[h]{0.34\textwidth}
    \includegraphics[width=\linewidth]{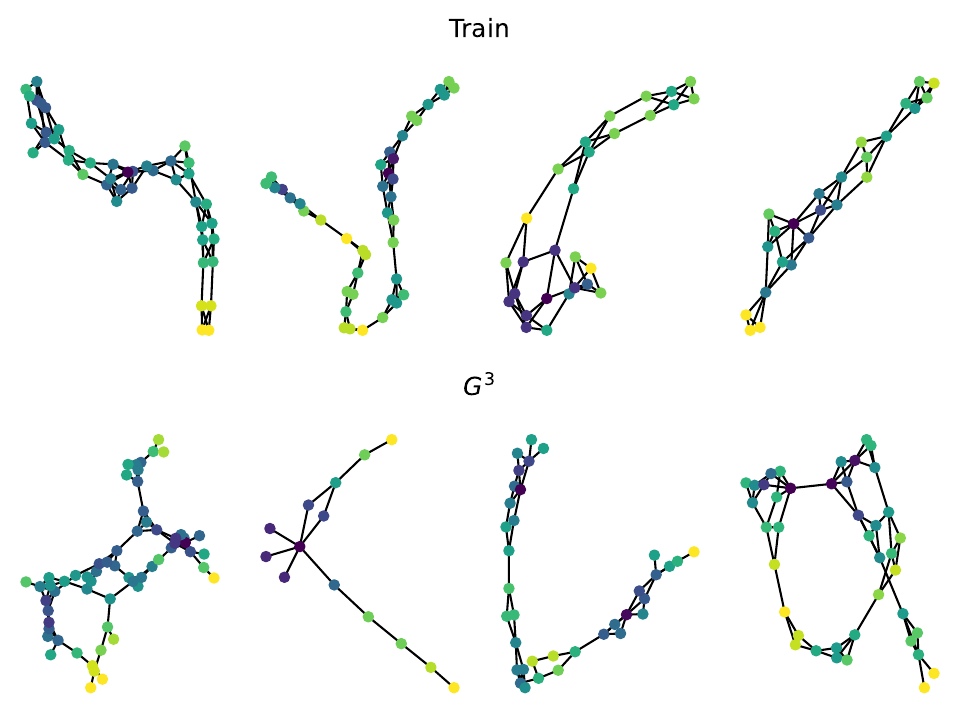}
    \caption{}\label{fig:example_enzyme_graphs}
    \end{subfigure}
    \caption{\subref{fig:forward_process} shows the effect of the symmetric heat diffusion process on a graph (bottom row) and its adjacency (top) as $t$ increases and the probability mass is spread over all possible edges. Colours indicate this value, with yellow and purple the respective extremes of the interval $[0,1]$ and edge transparency governed by the same quantity. \subref{fig:example_enzyme_graphs} shows a series of randomly drawn graphs from the (training) subset of the Enzyme dataset (top) and output graphs from $G^3$ (bottom). Node colouring reflects Forman-Ricci curvature \citep{forman2003bochner} at that node.}
    \label{fig:forward_process_and_example}
\end{figure*}

\paragraph{Contributions.} In this paper we propose \emph{Generator-based Graph Generation} ($G^3$), a generator-matching approach to graph generative modelling in which we (a) choose a principled, topology-aware \emph{forward} noising mechanism given by Laplacian heat diffusion on a matrix representation of a graph, (b) exploit its closed-form infinitesimal generator to define an explicit learning target, and (c) train a neural surrogate generator from forward-diffused samples. At sampling time, we integrate the learned reverse-time ODE from a simple, permutation-symmetric base distribution to obtain a real-valued graph matrix, which is then mapped to a discrete adjacency matrix. Figure~\ref{fig:forward_process} illustrates the forward-reverse behaviour of the $G^3$ framework.

\section{Background}
\label{sec:background}

Continuous-time generative models can be formulated in terms of deterministic or stochastic dynamics that transform samples from a simple reference distribution into samples from a complex data distribution. One of the most popular approaches for learning such dynamics is \emph{flow matching}, and its extension \emph{generator matching}. In addition to reviewing these paradigms in a general setting, we also establish the graph-theoretic preliminaries—matrix representations and Laplacian heat diffusion—that underpin our $G^3$ framework introduced in Section~\ref{sec:graph-generation}.

\subsection{Graph Matrix Representations}
\label{sec:graph-matrices}

We consider undirected simple graphs $G=(V,E)$ with $n=|V|$ nodes $V=\{1,\ldots,n\}$ and edge set $E\subseteq \{\{i,j\}: i,j\in V,\ i\neq j\}$. The adjacency matrix $A\in\{0,1\}^{n\times n}$ represents the connectivity in the graph where $A_{ij}=1$ iff $\{i,j\}\in E$, and $A_{ii}=0$ otherwise. The degree matrix $D=\mathrm{diag}(d_1,\ldots,d_n)$ is the diagonal matrix of node degrees $d_i=\sum_{j=1}^n A_{ij}$.

The combinatorial Laplacian $L = D - A$ is a symmetric positive semidefinite matrix that describes heat diffusion on an undirected graph. For a connected graph, the Laplacian has eigenvalue 0 with multiplicity 1 with the all-one eigenvector $\mathbf{1}$, $L\mathbf{1}=0$.

Other versions of the Laplacian matrix normalise the graph structure in different ways, e.g., the random-walk Laplacian $L_\text{rw} = I - D^{-1}A$. The following methodology applies to alternative Laplacian matrices but, by default, results refer to the combinatorial Laplacian unless stated otherwise.

\subsection{Heat Equation on Graphs and the Heat Kernel}
\label{sec:heat-kernel}

We consider heat diffusion on an undirected graph and introduce its solution, known as the \emph{graph heat kernel}. This construction provides a notion of diffusion on graphs that will be utilised in our generative framework. 

\begin{definition}[Heat equation on a graph]\label{def:heat-eq}
Let $G=(V,E)$ be an undirected graph with $n=|V|$ nodes, graph Laplacian $L \in \R^{n\times n}$ and heat kernel
\begin{equation}\label{eq:heat-kernel-def}
H_s := e^{-sL},\qquad s\ge 0.
\end{equation}
Given an initial matrix-valued graph representation $Y_0\in\R^{n\times n}$ (e.g.\ $Y_0=A$ the adjacency matrix), define the \emph{symmetric heat diffusion} by
\begin{equation}\label{eq:symmetric-heat-diffusion}
Y_s := H_s\,Y_0\,H_s,\qquad s\ge 0.
\end{equation}
\end{definition}

Since $\frac{\mathrm d}{\mathrm ds}H_s=-L H_s=-H_sL$, differentiating \eqref{eq:symmetric-heat-diffusion} yields the matrix-valued heat equation
\begin{equation}\label{eq:heat-matrix-ode}
\frac{\mathrm d}{\mathrm ds}Y_s
=
- L Y_s - Y_s L,
\qquad
Y_{0}=A.
\end{equation}
(See Proposition \ref{prop:heat-solution} in Appendix \ref{app:heat} for a proof.) As $L$ is symmetric positive semidefinite, the operator $Y\mapsto LY+YL$ is dissipative, and \eqref{eq:heat-matrix-ode} defines a stable linear diffusion on graph-derived matrices (see Lemmas \ref{lem:dissipative-LYplusYL} and \ref{lem:semigroup-contraction} in Appendix \ref{app:theory-dissipativity} and Figure \ref{fig:forward_process} for an illustration). Intuitively, left and right multiplication by $H_s$ smooths the structure encoded in $Y_0$ in a manner consistent with the graph topology.

Although \eqref{eq:symmetric-heat-diffusion} can be viewed via its associated semigroup and long-time limit, we defer that perspective to Section~\ref{sec:graph-diffusion-generator}, where it directly motivates our generator-matching formulation and the choice of base distribution.

\subsection{Flow Matching}
\label{sec:flow-matching}

Flow matching \citep{lipman2022flow} is a framework for generative modelling in which one learns a deterministic \emph{flow} that transports probability mass from a single, often simple, initial distribution to a target distribution over a continuous state space $\mathcal{X}\subseteq\mathbb{R}^d.$ The approach can be viewed as a deterministic analogue of diffusion-based generative modelling \citep{song2020score}, in which stochastic dynamics are replaced by ordinary differential equations (ODEs). 

\paragraph{Vector fields and probability paths.} Let $p_\text{init}$ denote a reference distribution on $\mathcal{X}$ (often a Gaussian or uniform distribution), and let $p_\text{data}$ denote the target data distribution. A flow matching model introduces a time-dependent vector field $u:\mathcal{X} \times [0,1] \to \mathcal{X}$ and defines trajectories $(X_t)_{t\in[0,1]}$ via the ODE
\begin{equation}
    \label{eq:flow-ode}
    \frac{\mathrm{d}}{\mathrm{d}t} X_t = u_t(X_t), \qquad X_0 \sim p_{\text{init}}.
\end{equation}
The solution of this ODE induces (under mild regularity conditions) a \emph{flow map} $\psi_t:\mathcal{X}\to\mathcal{X},$ such that $X_t=\psi_t(X_0).$ The goal is to learn a vector field $u$ for which the terminal distribution $X_1$ matches the target data distribution $p_{\text{data}}$.

Rather than directly specifying $u$ through a likelihood or score function, flow matching begins by prescribing a \emph{probability path} $(p_t)_{t\in[0,1]}$ interpolating between $p_{\text{init}}$ and $p_{\text{data}}$. A time-dependent vector field $u^*:\mathcal{X}\times[0,1]\to\mathcal{X}$ is said to be \emph{consistent} with this path if $(p_t)$ satisfies the continuity equation
\begin{equation}\label{eq:fm-continuity}
\partial_t p_t(x) + \nabla\cdot\big(p_t(x),u^*_t(x)\big)=0,
\quad t\in(0,1).
\end{equation}

\paragraph{Surrogate vector field and objective.} A neural vector field $u^\theta:\mathcal{X}\times[0,1]\to\mathcal{X}$ is then trained to approximate $u^*$ by minimising an expected discrepancy
\begin{equation}\label{eq:fm-loss}
\mathcal{L}{_\text{FM}}(\theta)=
\mathbb{E}_{t\sim\mathrm{Unif},x\sim p_t}\Big[D\big(u^*_t(x),u^\theta_t(x)\big)\Big],
\end{equation}
where $D(\cdot,\cdot)$ is a \emph{Bregman divergence}, typically the mean-squared error or KL divergence in practice.

 Flow matching avoids numerical issues commonly associated with simulating stochastic differential equations \citep{song2020score}, admits flexible choices of probability paths, and can be implemented using standard ODE solvers at sampling time \citep{chen2018neuralode}. However, the framework fundamentally operates at the level of \emph{vector fields}, and does not explicitly constrain the learned dynamics to correspond to any underlying stochastic process or infinitesimal generator. 

\subsection{Generator Matching}
\label{sec:generator-matching}

Generator matching \citep{holderrieth2024generator} is a framework for learning continuous-time generative models in which the focus shifts from matching vector fields to matching \emph{infinitesimal generators} of Markov processes. Let $(X_t)_{t\in[0,1]}$ be a (possibly \emph{time-inhomogeneous}) continuous-time Markov process on a state space $\mathcal{X}$ with (infinitesimal) generator $\mathcal{J}_t$. For sufficiently regular test functions $f:\mathcal{X}\to\mathbb{R}$ in the domain of $\mathcal{J}_t$, the generator is defined by
\begin{equation}\label{eq:gm-generator-def}
(\mathcal{J}_tf)(x) := \lim_{h\downarrow 0}\frac{\mathbb{E}\!\left[f(X_{t+h})\mid X_t=x\right]-f(x)}{h}.
\end{equation}

\paragraph{Semigroup and Kolmogorov equations.}
The generator viewpoint is tightly linked to the Kolmogorov equations describing how expectations and marginal laws evolve. In the \emph{time-homogeneous} case (i.e.\ when $\mathcal{J}_t\equiv \mathcal{J}$ does not depend on $t$), the process induces a one-parameter Markov semigroup $(S_t)_{t\ge 0}$ acting on observables $f:\mathcal{X}\to\mathbb{R}$ by
\[
(S_t f)(x) := \mathbb{E}\!\left[f(X_t)\mid X_0=x\right], \qquad t\ge 0,
\]
and for $f\in\mathrm{Dom}(\mathcal{J})$ this semigroup satisfies the Kolmogorov backward equation
\[
\partial_t(S_t f) = \mathcal{J}(S_t f), \qquad S_0 f = f.
\]
(Equivalently, $\partial_t(S_t f)=S_t(\mathcal{J}f)$.) Dually, if $p_t$ denotes the marginal law of $X_t$ and $\mathcal{J}^*$ is the adjoint of $\mathcal{J}$ in the weak sense, then $p_t$ evolves according to the forward (Fokker--Planck/master) equation
\[
\partial_t p_t = \mathcal{J}^* p_t .
\]
In the more general time-\emph{inhomogeneous} setting one works with a family $(\mathcal{J}_t)_{t\in[0,1]}$ and the corresponding evolution of $(S_t f)$ and $(p_t)$ is still governed by these backward/forward equations with $\mathcal{J}$ replaced by $\mathcal{J}_t$. This is precisely why generator matching is posed in terms of \emph{generator actions}: matching $\mathcal{J}_t$ (or $\mathcal{J}$) enforces consistency with the induced evolution of both observables and distributions.

\paragraph{Surrogate generator and objective.} Under generator matching, one assumes access to a forward process whose (possibly time-dependent) generator $\mathcal{J}_t$ induces a family of marginals $(p_t)_{t\in[0,1]}$ that interpolate between $p_{\text{init}}$ and $p_{\text{data}}$. 
Rather than matching vector fields as with flow matching, the objective is to approximate $\mathcal{J}_t$ by a parameterised surrogate generator, $\mathcal{J}^\theta_t$, by minimising a discrepancy between their \emph{generator actions} evaluated at forward samples. For a suitable class of test functionals $\mathcal{F}$, training takes the form of minimising the loss
\begin{equation}\label{eq:gm-loss-f}
\mathcal{L}_{\text{GM}}(\theta)
=
\mathbb{E}_{t\sim\mathrm{Unif},\, X_t\sim p_t}
\Big[
D\big((\mathcal{J}_t f)(X_t),(\mathcal{J}^\theta_t f)(X_t)\big)
\Big],
\end{equation}
for $f\in\mathcal{F}$ and where $D(\cdot,\cdot)$ denotes a Bregman divergence.

\paragraph{Linear functional parameterisation.} In structured-state settings that we consider for $G^3$ we choose $\mathcal{F}$ to be linear functionals $f_A(X)=\langle A,X\rangle$ (e.g.\ $\langle A,X\rangle=\mathrm{tr}(A^\top X)$ for matrix states). When there exists an induced operator on states, still denoted $\mathcal{J}_t$ by abuse of notation, such that both
\begin{align*}
(\mathcal{J}_t f_A)(X) &= \langle A,\,\mathcal{J}_t (X)\rangle, \\
(\mathcal{J}^\theta_t f_A)(X) & =\langle A,\,\mathcal{J}^\theta_t(X)\rangle,
\end{align*}
and the objective reduces to the state-space form
\begin{equation}\label{eq:gm-loss-state}
\mathcal{L}_{\text{GM}}(\theta)
=
\mathbb{E}_{t\sim\mathrm{Unif},\, X_t\sim p_t}
\Big[
D\big(\mathcal{J}_t (X_t),\mathcal{J}^\theta_t(X_t)\big)
\Big],
\end{equation}
which we will adopt throughout Section~\ref{sec:graph-generation}. 

Flow matching can be viewed as a special case of generator matching when the generator corresponds to a deterministic ODE and the divergence is chosen appropriately. However, the generator-centric perspective is more expressive: it captures stochastic dynamics, jump processes, and structured state spaces, and naturally enforces consistency with the Kolmogorov equations governing the underlying process. This makes generator matching particularly appealing to structured data such as graphs, where a meaningful forward process can be directly defined in terms of graph operators. This observation forms the basis of our $G^3$ method.

\section{Graph Generation}
\label{sec:graph-generation}

We now introduce our generator-based graph generation method $G^3$. The construction starts from a graph noising mechanism given by a symmetric Laplacian heat diffusion on matrix-valued graph representations $(Y_s)_{s\in[0,T]}$ (see Section~\ref{sec:heat-kernel}), and converts this into a $[0,1]$-parametrised process $(X_t)_{t\in[0,1]}$ compatible with generator matching (Section~\ref{sec:generator-matching}) via the time-rescaling $X_t := Y_{T(1-t)}$. This yields an explicit ground-truth generator on graph states, which we approximate by a neural surrogate $\mathcal{J}^\theta_t$ using a generator matching objective. We then sample new graphs by numerically integrating the learned reverse-time ODE from a simple base distribution.

\subsection{Graph Diffusions and Generators}
\label{sec:graph-diffusion-generator}

We embed the graph heat diffusion into the generator matching framework of Section~\ref{sec:generator-matching}. The key point is that the symmetric heat diffusion admits a closed-form generator; by rescaling time, we obtain a process on $[0,1]$ consistent with the $X$-notation used in Sections~\ref{sec:flow-matching}--\ref{sec:generator-matching}.

\paragraph{Semigroup viewpoint and long-time limit.}
The forward heat diffusion \eqref{eq:symmetric-heat-diffusion} admits a convenient semigroup interpretation that
connects naturally to generator matching and the choice of a simple base state.
Define the linear maps, for $s\ge 0$,
\begin{equation}\label{eq:heat-semigroup}
\mathcal{P}_s(Y) := H_s\,Y\,H_s.
\end{equation}
Then $\{\mathcal{P}_s\}_{s\ge 0}$ is a strongly continuous semigroup on $\R^{n\times n}$, with
$\mathcal{P}_0=\mathrm{Id}$ and $\mathcal{P}_{s+t}=\mathcal{P}_s\circ\mathcal{P}_t$.
Moreover, its infinitesimal generator is
\[
(\mathcal{G}Y) = -\big(LY+YL\big),
\]
so that \eqref{eq:heat-matrix-ode} can be written abstractly as
$\frac{\mathrm d}{\mathrm ds}Y_s = \mathcal{G}Y_s$ (see Lemma~\ref{lem:semigroup} in Appendix~\ref{app:semi-generators}).

\paragraph{Limiting behaviour and probabilistic interpretation.}
For a connected graph, the Laplacian has a simple eigenvalue at $0$ with eigenvector $\mathbf 1$, and all other
eigenvalues are strictly positive; hence
\[
\lim_{s\to\infty} H_s = \Pi := \frac{1}{n}\mathbf 1 \mathbf 1^\top,
\]
the orthogonal projector onto constants. Consequently,
\[
\lim_{s\to\infty} Y_s
= \Pi Y_0 \Pi
= \frac{1}{n^2} \| Y_0 \|_1 \,\mathbf 1\mathbf 1^\top,
\]
for non-negative matrices $Y_0$. For connected graphs, this convergence is exponentially fast, with rate governed by the spectral gap of the Laplacian (see Lemma~\ref{lem:frob-norm} in Appendix~\ref{app:forward-reverse}).
Moreover, when using the random-walk Laplacian $L_\text{rw}$, $H_s$ admits the interpretation as the transition kernel of a continuous-time random walk on $V$ \citep{chung1997spectral}.

\paragraph{Time rescaling to $[0,1]$.}
Generator matching (and flow matching) are formulated on $t\in[0,1]$ with $X_0$ drawn from a simple reference distribution and $X_1\sim p_\text{data}$. To align with this convention we define, for a fixed horizon $T>0$,
\begin{equation}
\label{eq:Xdef}
X_t := Y_{T(1-t)},\qquad t\in[0,1].
\end{equation}
Thus, $X_0=Y_T$ is a highly diffused (simple) state, while $X_1=Y_0$ recovers the original graph representation.
Differentiating \eqref{eq:Xdef} yields
\begin{equation}
\label{eq:X-ode}
\frac{\mathrm d}{\mathrm dt}X_t
=
T\big(LX_t + X_t L\big),\qquad t\in[0,1].
\end{equation}
Equivalently, in terms of the forward generator $\mathcal G$ above, the reverse-time operator satisfies $\mathcal J = -T\,\mathcal G$.
The forward heat diffusion is contractive; its formal time-reversal corresponds to an anti-diffusive evolution that exponentially amplifies high-frequency (Laplacian eigenmode) components, rendering the exact reverse dynamics ill-posed and numerically unstable \citep{Hadamard23,Evans10} (see Lemma~\ref{lem:frob-norm-rev} in Appendix~\ref{app:forward-reverse}).

Our generator-matching approach circumvents this instability by learning a regularised surrogate generator $\mathcal{J}^\theta_t$ from forward-diffused samples, yielding a stable and well-conditioned reverse-time generative sampler.

\paragraph{Infinitesimal generator on states.}
For linear test functionals $f_A(X)=\langle A,X\rangle:=\mathrm{tr}(A^\top X)$ with $A\in\R^{n\times n}$, \eqref{eq:X-ode} implies
\[
\frac{\mathrm d}{\mathrm dt} f_A(X_t)
= \Big\langle A,\frac{\mathrm d}{\mathrm{d}t} X_t\Big\rangle=
\left\langle A,\; T(LX_t+X_tL)\right\rangle.
\]
Restricting attention to such linear functionals induces a (state-space) time-homogeneous linear operator, which we denote by $\mathcal{J}$, defined through
\begin{equation}
\label{eq:true-J-state}
(\mathcal{J} f_A)(X) =\langle A, \mathcal{J}(X)\rangle = \langle A, T\big(LX + XL\big)\rangle.
\end{equation}
 This $\mathcal{J}$ is fully determined by the graph Laplacian and therefore respects permutation equivariance and graph locality by construction.

\subsection{Graph Generator Training}
\label{sec:training}

\paragraph{Training data and forward diffusion.}
Let $\{Y^{(k)}_0\}_{k=1}^N$ denote a set of training data graphs, represented as matrices. 
Define the symmetric heat diffusion
\[
Y^{(k)}_s = e^{-sL^{(k)}}\,Y^{(k)}_0\,e^{-sL^{(k)}},\qquad s\in[0,T],
\]
and the time-rescaled process on $[0,1]$,
\[
X^{(k)}_t := Y^{(k)}_{T(1-t)},\qquad t\in[0,1].
\]
Noised graph training samples are obtained by drawing $t\sim\mathrm{Unif}[0,1]$ and forming $X^{(k)}_t$ via the above mapping. Throughout, we take the initial graph state to be the adjacency matrix $Y^{(k)}_0=A^{(k)}$ with associated Laplacian $L^{(k)}$. Since all of these objects are symmetric, we retain only the lower-triangular parts of matrices in our implementation, which implicitly removes the presence of any self-loops.

\paragraph{Surrogate generator and objective.}
From \eqref{eq:true-J-state}, the true generators for training graphs $\{Y^{(k)}_0\}_{k=1}^N$ acting on states are
\begin{equation}
\label{eq:true-J-k}
\mathcal{J}^{(k)}(X) = T\big(L^{(k)}X + X L^{(k)}\big).
\end{equation}

Using these target generators, we learn a surrogate generator $\mathcal{J}^\theta:\R^{n\times n} \times [0,1] \to\R^{n\times n}$ by minimising
\begin{equation}
\label{eq:gm-objective}
\E_{k}\E_{t\sim\mathrm{Unif}}
\Big[
D\big(
\mathcal{J}^{(k)}(X^{(k)}_t),\;
\mathcal{J}^\theta_t(X^{(k)}_t)
\big)
\Big],
\end{equation}
where $D$ is a Bregman divergence. In the experiments described in Section~\ref{sec:experiments}, we exclusively use mean-squared error loss, $D(A,B)=\|A-B\|_F^2$. See Algorithm \ref{alg:training}.

\subsection{Sampling via the Reverse-Time ODE}
\label{sec:sampling}

Having learned a surrogate generator $\mathcal{J}^\theta_t$ from the forward-noised graphs $\{Y_0^{(k)},Y_1^{(k)},\ldots,Y_T^{(k)}\}_{k=1}^N$, we generate new graphs under the time-rescaled setting by integrating the learned ODE from a simple base distribution at $t=0$ toward a structured graph state at $t=1$. Note that $t\mapsto X_t=Y_{T(1-t)}$ already reverses the original diffusion time $s$, so integrating forward in $t$ corresponds to a formal reversal of the heat diffusion in $s$.

\paragraph{Base distribution.}
We take $p_{\mathrm{init}}$ to be a permutation-symmetric distribution intended to approximate a highly diffused state $X_0=Y_T$ induced by the forward heat diffusion (cf.\ Section~\ref{sec:heat-kernel}). Concretely, we sample an auxiliary matrix $\widetilde X_0\in\R^{n\times n}$ by drawing its columns i.i.d as $\bm{\tilde x}_{0,i}\sim \mathrm{Dirichlet}(\alpha \mathbf{1})$, $i=1,\ldots,n$, enforcing symmetry, and scaling by the average degree $\avg{d}:=\tfrac2{nN}\sum_{k=1}^N \| Y^{(k)}_0\|_1$:
\begin{equation}
X_0 := \avg{d}(\widetilde X_0+\widetilde X_0^\top),
\end{equation}
with $\alpha>0$ controlling concentration around the uniform vector. While the exact law of $Y_T$ under the data distribution is intractable, $p_{\mathrm{init}}$ provides a simple surrogate capturing permutation symmetry and strong diffusion.

\paragraph{Simulating an ODE on $[0,1]$.}
For each training graph $k$, the time-rescaled heat dynamics satisfy
$\frac{\mathrm d}{\mathrm dt}X^{(k)}_t = \mathcal{J}^{(k)}(X^{(k)}_t)$ with $\mathcal{J}^{(k)}(X) = T(L^{(k)}X+XL^{(k)})$.
At generation time we do not condition on a specific Laplacian; instead we integrate the learned surrogate dynamics
\begin{equation}
\label{eq:learned-sampling-ode}
\frac{\mathrm d}{\mathrm dt}X_t = \mathcal{J}^\theta_t(X_t),\qquad t\in[0,1],\qquad X_0\sim p_{\mathrm{init}},
\end{equation}
where $\mathcal{J}^\theta_t$ has been trained to match the heat-generator action from the forward-diffused states.

\paragraph{Euler discretisation and stabilisation.}
Fix an integer $M\ge 1$ and set $\delta:=1/M$, with time grid $t_m := m\delta$ for $m=0,\ldots,M.$ Approximating $X_{t_m}$ by $X^m$, explicit Euler integration applied to \eqref{eq:learned-sampling-ode} yields 
\begin{equation}
\label{eq:euler}
X^{m+1} = X^{m} + \delta\,\mathcal{J}^\theta_t(X^{m}),\qquad m=0,\dots,M-1.
\end{equation}
To enforce basic validity constraints and improve numerical stability, we optionally apply a projection $\Pi$ after each step:
\begin{equation}
\label{eq:euler-proj}
X^{m+1} = \Pi\!\left(X^{m} + \delta\,\mathcal{J}^\theta_t(X^{m})\right).
\end{equation}
Typical projections include symmetrisation $\Pi_{\mathrm{sym}}(Z)=\tfrac12(Z+Z^\top)$, zeroing the diagonal, and element-wise clipping. The final discrete adjacency matrix is obtained by a thresholding operation, $\widehat A_{ij}=\mathbf{1}\{(\Pi_{\mathrm{sym}}X^{M})_{ij}\ge c\}$.

\begin{figure*}[t]
\begin{minipage}[t]{0.48\textwidth}
\begin{algorithmbox}[$G^3$ Training]
  \label{alg:training}
  \begin{algorithmic}
    \STATE {\bfseries Input:} Graph representations $\{Y_0^{(j)}\}_{j=1}^N$, maximum diffusion time $T$, batch size $b$, $n_{\mathrm{iter}}$, $\delta > 0$, learning rate $\eta$, minimum diffusion time, $\tau:=0.01$
    \REPEAT
    \FOR{batch {\bfseries in} batches}
    \STATE Sample diffusion times: $\{t_i\}_{i=1}^b \sim \mathrm{Unif}[\tau,T]$
    \STATE Diffuse states: $\{X^{(i)}_{s_i}\}_{i=1}^b$ for $s_i:=1-t_i/T$
    \STATE Generator targets: $\{T(L_iX^{(i)}_{s_i} + X^{(i)}_{s_i}L_i)\}_{i=1}^b$
    \STATE Predict generators: $\{\mc{J}_{t_i}^\theta(X^{(i)}_{s_i})\}_{i=1}^b$
    \STATE $\mathrm{loss}=\frac{1}{b}\sum_{i=1}^b\norm{\mc{J}_{t_i}^\theta(X^{(i)}_{s_i}) - \mc{J}_{t_i}(X^{(i)}_{s_i})}_F^2$
    \IF{$\mathrm{loss}$ unchanged after 10 iterations}
        \STATE $\eta\leftarrow \max\{\eta \times 0.99, \, 10^{-9}\}$
    \ENDIF
    \ENDFOR
    \UNTIL{$\mathrm{loss}< \delta$ or $n_{\mathrm{iter}}$ reached}
  \end{algorithmic}
\end{algorithmbox}
\end{minipage}
\hfill
\begin{minipage}[t]{0.48\textwidth}
  \begin{algorithmbox}[$G^3$ Sampling]
  \label{alg:sampling}
  \begin{algorithmic}
    \STATE {\bfseries Input:} Number of new graphs $N^*$, maximum diffusion time $T$, number of Euler steps $M$, base distribution concentration $\alpha>0$
    \STATE $\delta \leftarrow 1/M$
    \STATE $c\leftarrow \frac{1}{n^2N}\sum_{j=1}^N\|A^{(j)}\|_1$ \\
    \hfill \COMMENT{Sparsity threshold using average degree}
    \FOR{i=0 {\bfseries to} $N^*$}
    \STATE Sample $X^{(i)}_0 \sim p_{\mathrm{init}}^\alpha$
    \FOR{k=0 {\bfseries to} $M$}
        \STATE $X^{(i)}_{k+1} \leftarrow \Pi_{\mathrm{clip}}(X^{(i)}_k + \delta \mc{J}_t^\theta(X^{(i)}_k))$ \\
        \hfill \COMMENT{Euler steps with element clipping}
    \ENDFOR
            \STATE $\hat{A}^{(i)}_{ab} \leftarrow \1\{(\Pi_\text{sym} (X^{(i)}_M)_{ab} \geq c\}$ \\
            \hfill \COMMENT{Symmetrisation; diagonal zeroing; thresholding}
    \ENDFOR
  \end{algorithmic}
\end{algorithmbox}
\end{minipage}
\end{figure*}

\section{Related Work}
\label{sec:related-work}

Beyond the work discussed in \cref{sec:intro}, we highlight several representative graph generation methods. SPECTRE \citep{martinkus2022spectre} trains a sequence of GANs to sample eigenvalues and then reconstruct eigenvectors and an adjacency matrix; the innovation is that spectra encode the salient global topology. 

Recent progress has also focused on discrete adaptations of diffusion and flow-matching ideas, including DiGress \citep{vignac2022digress} and DeFoG \citep{qin2024defog}. DeFoG achieves strong empirical performance with theoretical guarantees but requires substantial compute, especially on larger graphs, due to its transformer architecture and relative random walk probability features (RRWP).

We use SPECTRE and DeFoG as high-quality, general-purpose baselines. To keep the study computationally tractable, we restrict attention to a small set of competitors and report results from our own reruns of their code, rather than quoting results from earlier papers, to support fair comparisons and reproducibility. 

Recent work by \citep{law2025directed} explores the use of heat kernels for graph generation in a related but fundamentally different setting. Their approach targets \emph{directed} graphs and relies on a stochastic heat kernel acting on learned node representations, coupled with a separate edge model that maps these representations back to a valid graph. In contrast, our framework defines and explicitly reverses a diffusion process directly on graph-derived matrix states via generator matching, without introducing an auxiliary edge reconstruction mechanism. As a result, our method applies equally to undirected graphs and, in principle, can be extended to directed settings within the same generator-centric formulation. Moreover, the edge perturbation and reconstruction steps employed in \cite{law2025directed} may alter or obscure global structural constraints, such as planarity, which our Laplacian-based diffusion preserves by construction.

\section{Experiments}
\label{sec:experiments}

\subsection{Datasets}
\label{sec:datasets}
To assess $G^3$ performance, we consider three synthetic (SBM, DCSBM, Planar) and three real-world graph datasets (Enzymes, Proteins, QM9), many of which are commonly used for benchmarking in the graph generation literature. The exception is the degree-corrected stochastic block model (DCSBM) dataset, which allows for more reasonable models of community-structured networks \citep{karrer2011stochastic}.
Details are provided in \cref{appx:datasets}.

\subsection{Evaluation}
\label{sec:evaluation}
To evaluate $G^3$ graph generation performance, we rely on the maximum mean discrepancy (MMD) of various graph attributes between hold-out samples of graphs and our generated samples, as is standard in the literature. Effectiveness is assessed by comparing the attribute distribution between hold-out test graphs and generated graphs for the following five graph characteristics: clustering coefficient, degree distribution, 4-node orbit count, graph Laplacian spectrum and triangle counts.
Further details are provided in \cref{appx:evaluation}.

\subsection{Neural Network Architecture}
\label{sec:architecture}
The $G^3$ generator surrogate $\mathcal{J}^\theta_t$ trained using \cref{alg:training} is a simple 4-layer MLP neural network with ReLU activation functions and layer normalisation. The network acts on the flattened form of the diffused network input state that is reshaped into a square matrix for output. In general, there are 4096 units per hidden layer, where the sensitivity of this design choice is explored in \cref{sec:behaviour}.
Further details are provided in \cref{appx:implementation}.

\subsection{Graph Generation Sensitivity}
\label{sec:behaviour}
To understand how sensitive $G^3$ is to graph topology, training dynamics and hyperparameters, we consider how effective the methodology is as these factors are systematically varied. 
We focus on the spectral MMD as the most comprehensive single metric over which to assess sensitivity. Additional figures are provided in \cref{appx:experiments}.

\paragraph{Sensitivity to maximum diffusion time.} \cref{fig:spec_v_tmax_by_N_planar} shows an example of how spectral MMD is sensitive to the choice of $T$, the maximum time states are diffused to during the forward training process and its dependence on the number of graphs, $N$. This provides us with a simple general principle that $T$ should increase with training set size, avoiding the need to treat this as a tunable hyperparameter.


\cref{fig:spec_v_tmax_by_n_planar} demonstrates less dependence on the number of nodes, $n$, than might have been expected; we can broadly opt to neglect its influence on this same choice of parameter. The minimum of the curves seems to occur at approximately the same $T$, indicating that the number of graphs, $N$, is the more important factor and that the problem gets ``harder'' as $n$ increases, at least for planar graphs.

\paragraph{Sensitivity to neural network architecture.} \cref{fig:spec_v_w_by_N_planar} shows an example of how spectral MMD is sensitive to the choice of the number of hidden neural network units per layer of the MLP, $w$, and its dependence on the number of graphs, $N$. The minimum occurs at approximately $2^{12}=4096$ hidden units, independent of $N$.

Based on this independence, \cref{tab:sbm_w} computes a small table for SBMs with fixed $N$ for use across all other datasets in our further experiments to save unnecessary computation. In other words, we do not tune the neural network architecture for each dataset in turn.




\paragraph{Sensitivity to base distribution.} \cref{fig:spec_v_alpha} shows how the choice of hyperparameter $\alpha$ in the base sampling distribution affects the efficacy of the model for SBM and planar graphs. The choice of $\alpha$ not only affects the spectral performance but also the ability to generate unique new graphs (see \cref{fig:nonu_v_alpha}).

Empirically, $\alpha < 1$ is preferable for planar-like graphs including the biomolecular datasets, while for community-like graphs such as SBM and DCSBM, $\alpha=1$ is preferable.



\newcounter{ribbontextmark} 
\newcommand{\SmoothRibbonText}{%
\,Curves are Gaussian-smoothed\!
  \ifnum\value{ribbontextmark}=0
    \protect\footnotemark
  \else
    \protect\footnotemark[\value{ribbontextmark}]
  \fi
   \,functions of the observations shown, which themselves are averaged over 10 random seeds. Shaded regions indicate 66\% confidence intervals under CLT assumptions. Learning rate set at $1\times 10^{-4}$ for training.
}

\newcommand{\ribbonfootnote}{%
  \ifnum\value{ribbontextmark}=0
    \footnotetext{Using \texttt{gaussian\_filter1d} function from SciPy package.}%
    \setcounter{ribbontextmark}{\value{footnote}}%
  \fi
}

\def\FigureScale{0.4}

\begin{figure}[t]
    \centering
    \includegraphics[scale=\FigureScale]{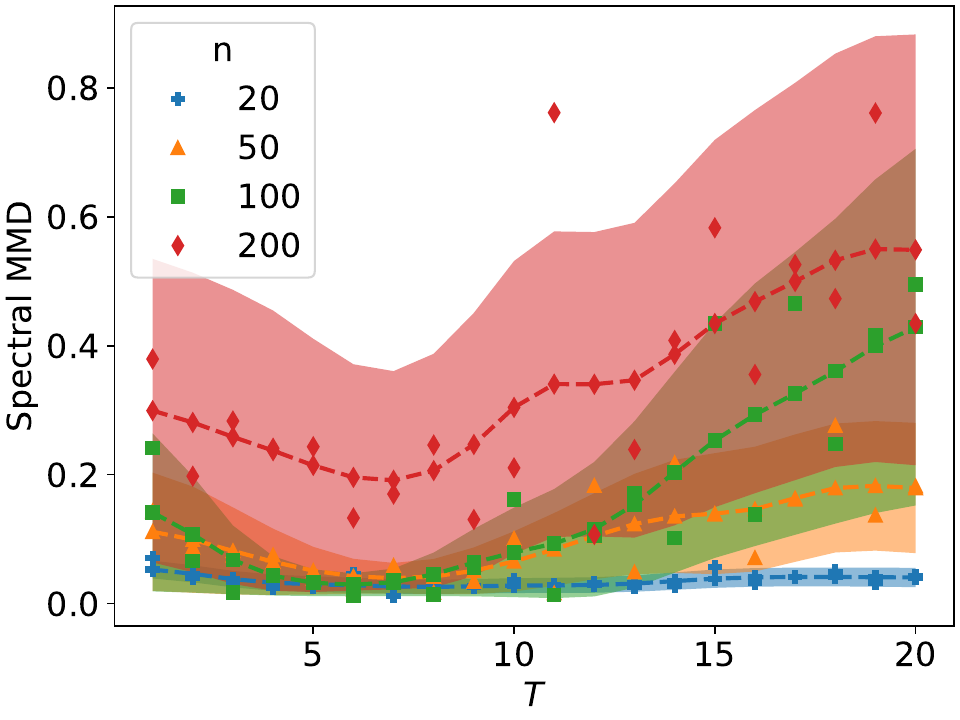}
    \caption{Spectral MMD as a function of maximum diffusion time $T$ and the number of planar graphs $N$ with $n=64$ nodes.\SmoothRibbonText }
    \label{fig:spec_v_tmax_by_N_planar}
\end{figure}
\ribbonfootnote

\begin{figure}[t]
    \centering
    \includegraphics[scale=\FigureScale]{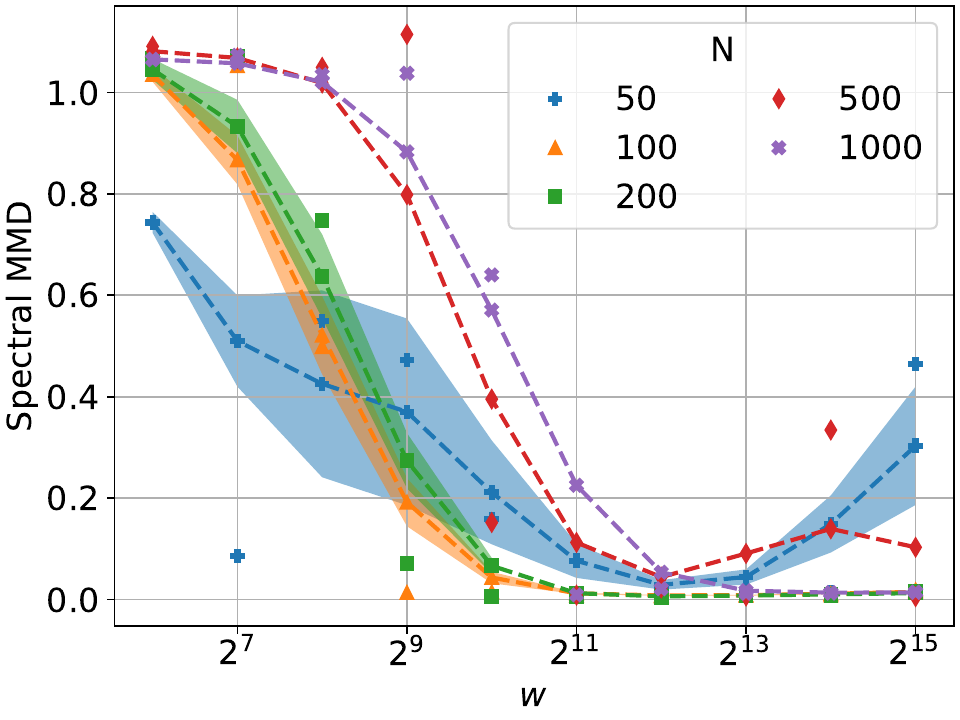}
    \caption{Spectral MMD as a function of the number of hidden neural network units per layer $w$ and the number of planar graphs $N$ with $n=64$ nodes.
    \SmoothRibbonText}
    \label{fig:spec_v_w_by_N_planar}
\end{figure}

\subsection{Performance}
\label{sec:performance}

\begin{table*}[h!]\footnotesize
\centering
 \caption{Median (± median abs.~deviation) of MMD metrics for $G^3$, DeFoG, SPECTRE over $3-5$ randomly seeded splits. Lower is better and \textbf{bold} indicates best performance; \textit{italics} second-best performance, based on the median plus one median abs.~deviation.}
\label{tab:all_models_train_mmd}
\begin{tabularx}{0.95\linewidth}{llccccc}
\toprule
Dataset & Model & Clustering & Degree & Orbit & Spectrum & Triangles \\
\midrule
\multirow[c]{4}{*}{SBM}
    & $G^3$        & \bfseries 0.0356 ± 0.0005 & \bfseries 0.0318 ± 0.0052 & 0.0631 ± 0.0120 & \itshape 0.0796 ± 0.0120 & \itshape 0.0270 ± 0.0078 \\
    & Asymm.~$G^3$ & \itshape 0.0398 ± 0.0010 & \itshape 0.0390 ± 0.0052 & 0.0710 ± 0.0140 & \bfseries 0.0832 ± 0.0015 & 0.0287 ± 0.0084 \\
    & DeFoG        & 0.0500 ± 0.0000 & 0.0258 ± 0.0190 & \itshape 0.0500 ± 0.0000 & 0.0884 ± 0.0740 & \bfseries 0.0097 ± 0.0083 \\
    & SPECTRE      & 0.0588 ± 0.0089 & 0.2030 ± 0.1700 & \bfseries 0.0500 ± 0.0000 & 0.7010 ± 0.4800 & 0.3320 ± 0.2300 \\
\midrule
\multirow[c]{4}{*}{DCSBM}
    & $G^3$        & \itshape 0.1270 ± 0.0091 & \itshape 0.0344 ± 0.0062 & \itshape 0.0971 ± 0.0057 & 0.1230 ± 0.0330 & \itshape 0.0110 ± 0.0098 \\
    & Asymm.~$G^3$ & 0.2340 ± 0.0160 & 0.0810 ± 0.0200 & 0.1420 ± 0.0170 & \itshape 0.1010 ± 0.0220 & 0.0578 ± 0.0250 \\
    & DeFoG        & \bfseries 0.0556 ± 0.0067 & \bfseries 0.0063 ± 0.0050 & \bfseries 0.0591 ± 0.0110 & \bfseries 0.0499 ± 0.0240 & \bfseries 0.0006 ± 0.0002 \\
    & SPECTRE      & 1.0800 ± 0.0340 & 0.6640 ± 0.0007 & 1.0400 ± 0.0004 & 1.2500 ± 0.0250 & 0.7340 ± 0.0130 \\
\midrule
\multirow[c]{4}{*}{Planar}
    & $G^3$        & 0.3090 ± 0.0300 & \itshape 0.0030 ± 0.0015 & \itshape 0.0035 ± 0.0018 & \itshape 0.0060 ± 0.0015 & 0.0467 ± 0.0360 \\
    & Asymm.~$G^3$ & 1.0200 ± 0.0440 & 0.0074 ± 0.0050 & 0.0106 ± 0.0089 & 0.0113 ± 0.0021 & 0.1110 ± 0.0320 \\
    & DeFoG        & \bfseries 0.0413 ± 0.0047 & \bfseries 0.0007 ± 0.0001 & \bfseries 0.0002 ± 0.0000 & \bfseries 0.0006 ± 0.0003 & \bfseries 0.0056 ± 0.0013 \\
    & SPECTRE      & \itshape 0.1420 ± 0.0220 & 0.0242 ± 0.0190 & 0.4240 ± 0.3800 & 0.0560 ± 0.0019 & \itshape 0.0226 ± 0.0026 \\
\midrule
\multirow[c]{4}{*}{Enzymes}
    & $G^3$        & \bfseries 0.0225 ± 0.0008 & \itshape 0.0854 ± 0.0550 & 0.0902 ± 0.0390 & \itshape 0.0755 ± 0.0220 & \itshape 0.0490 ± 0.0380 \\
    & Asymm.~$G^3$ & 0.0803 ± 0.0200 & 0.0934 ± 0.0520 & \itshape 0.1040 ± 0.0047 & 0.1090 ± 0.0480 & 0.0752 ± 0.0120 \\
    & DeFoG        & 0.0765 ± 0.0036 & 0.7220 ± 0.0070 & 0.5530 ± 0.0140 & 1.2000 ± 0.0270 & 0.4650 ± 0.0033 \\
    & SPECTRE      & \itshape 0.0275 ± 0.0002 & \bfseries 0.0168 ± 0.0013 & \bfseries 0.0796 ± 0.0120 & \bfseries 0.0490 ± 0.0030 & \bfseries 0.0236 ± 0.0029 \\
\midrule
\multirow[c]{4}{*}{Proteins}
    & $G^3$        & \bfseries 0.0465 ± 0.0009 & \bfseries 0.0220 ± 0.0012 & \itshape 0.5530 ± 0.0720 &\itshape 0.1600 ± 0.0074 & \itshape 0.0928 ± 0.018 \\
    & Asymm.~$G^3$ & \itshape 0.0750 ± 0.0073 & \itshape 0.0345 ± 0.0003 & \bfseries 0.3480 ± 0.2000 & \bfseries 0.0921 ± 0.0110 & \bfseries 0.0419 ± 0.0071 \\
    & DeFoG        & 0.1050 ± 0.0032 & 0.7500 ± 0.0003 & 1.2700 ± 0.0000 & 1.2000 ± 0.0110 & 0.5800 ± 0.0006 \\
    & SPECTRE      & 0.1070 ± 0.0001 & 0.0867 ± 0.0480 & 0.8400 ± 0.0085 & 0.1310 ± 0.0430 & 0.1240 ± 0.0110 \\
\midrule
\multirow[c]{4}{*}{QM9}
    & $G^3$        & 0.2530 ± 0.0074 & 0.1870 ± 0.0400 & 0.1650 ± 0.0330 & 0.1100 ± 0.0027 & 0.1100 ± 0.0066 \\
    & Asymm.~$G^3$ & \itshape 0.0834 ± 0.0170 & \itshape 0.0214 ± 0.0033 & 0.0036 ± 0.0013 & \itshape 0.0168 ± 0.0033 & \itshape 0.0521 ± 0.0130 \\
    & DeFoG        & 0.5460 ± 0.3000 & 0.0449 ± 0.0170 & \itshape 0.0036 ± 0.0001 & 0.0640 ± 0.0049 & 0.1730 ± 0.0950 \\
    & SPECTRE      & \bfseries 0.0200 ± 0.0110 & \bfseries 0.0027 ± 0.0001 & \bfseries 0.0008 ± 0.0001 & \bfseries 0.0046 ± 0.0003 & \bfseries 0.0051 ± 0.0043 \\
\bottomrule
\end{tabularx}
\end{table*}

We now compare the performance of $G^3$, using both symmetric and asymmetric formulations of the heat diffusion strategy, to both SPECTRE and DeFoG over a selection of synthetic and real-world datasets. The asymmetric version of $G^3$ is outlined in \cref{appx:asymmetric}.

We compute the median $\pm$ median absolute deviation of the MMD between hold-out test graphs and generated graphs reported in \cref{tab:all_models_train_mmd}, where lower quantities are better. Although there is no clear winner over all of the datasets and metrics, $G^3$ is not only competitive, but \emph{substantially} faster to run. Training and generation time per sample graph takes seconds for $G^3$ compared to 15-30 minutes for DeFoG and SPECTRE (see \cref{fig:train_times_v_n}).

To highlight this point, we repeated the DeFoG and SPECTRE experiments on datasets where they previously performed well, but under reduced computational budgets. The number of training epochs was decreased by a factor of 10-100, although their total training times still generally exceeded those of $G^3$. We denote these constrained variants by DeFoG* and SPECTRE* and report the results in \cref{tab:short_models_train_mmd}. Under these conditions, $G^3$ tends to outperform the alternative methods, particularly with respect to degree distribution and spectral accuracy.

Although on the whole the symmetric heat diffusion leads to better performance than the asymmetric version, there is no great margin, and on certain datasets (QM9, Enzymes and Proteins) the asymmetric process is preferred. One might suppose that biomolecular structures lend themselves more to the asymmetric approach, although the superior performance of the symmetric version on the planar synthetic dataset effectively rules out planarity as an explanation.  

\section{Discussion}
In this work, we have introduced a novel, well-grounded framework for graph generation by integrating symmetric heat diffusion with generator matching. The use of the graph Laplacian as the generator injects domain-specific inductive bias, ensuring that the noising process inherently respects graph topology (e.g., local connectivity and geometry). Empirically, we have shown, via controlled simulation studies and real-world graph generation benchmarks, that our method is competitive and highly scalable.


An appealing aspect of this methodology is the number of avenues for further development. Our general framework is permutation-invariant in construction, but our current implementation uses a simple MLP for generator matching. While the experiments show success in this simple setting, we believe this could be further improved by using a permutation-equivariant neural network architecture such as GNNs.

The theoretical framework outlined in \cref{sec:training}--\ref{sec:sampling} relies on the heat kernel associated with the combinatorial graph Laplacian. Alternatives
such as the random walk Laplacian could be used, as the underlying $G^3$ framework is more general (as motivated in \cref{appx:method}).

Relatedly, we focus on a symmetric form of heat diffusion on undirected graphs, but by relaxing the symmetry assumption we could extend this work to directed graphs. In a similar vein, we expect other more exotic structures, including hypergraphs, to be amenable to the $G^3$ strategy too via a suitable hypergraph Laplacian.

An extension of particular interest to the authors is to incorporate graph and node covariate information for conditional generation, known as ``guidance'' in the generative literature. A toy example of how this might work is provided in \cref{appx:conditional_generation}, where we generate sample SBMs with pre-specified group membership.


\pagebreak

\section*{Acknowledgements}
 CN gratefully alcknowledges the support of EPSRC grants EP/V022636/1 and EP/Y028783/1. This work was also supported by Research England under the Expanding Excellence in England (E3) funding stream, which was awarded to MARS: Mathematics for AI in Real-world Systems in the School of Mathematical Sciences at Lancaster University. 

\section*{Impact Statement}
This paper presents a broad and adaptable framework for generating graphs, aiming to support a wide range of applications that include drug discovery, general graph synthesis, circuit design and cybersecurity. Advances in graph generation hold the potential to influence many scientific and technological fields and may carry both beneficial and challenging societal or ethical implications, particularly within biomedical and chemical research. However, at this stage, we do not anticipate any immediate societal concerns related to the methods introduced here.





\bibliography{refs}
\bibliographystyle{icml2026}

\newpage
\appendix
\onecolumn

\section{Proofs and Further Details}
\label{app:proofs}

\subsection{Heat Equation and Heat Kernel}
\label{app:heat}

\begin{proposition}[Solution of the symmetric heat diffusion via the heat kernel]\label{prop:heat-solution}
Let $L\in\mathbb{R}^{n\times n}$ be the (combinatorial) Laplacian of an undirected graph, and define the heat kernel $H_s := e^{-sL}$ for $s\ge 0$. For any initial matrix $Y_0\in\mathbb{R}^{n\times n}$, the matrix-valued symmetric heat equation
\begin{equation}\label{eq:heat-matrix-ode-2}
\frac{\mathrm d}{\mathrm ds}Y_s \;=\; -\big(LY_s + Y_sL\big),
\qquad s\ge 0,
\end{equation}
admits the unique solution
\begin{equation}\label{eq:heat-matrix-solution}
Y_s \;=\; H_s\,Y_0\,H_s \;=\; e^{-sL}\,Y_0\,e^{-sL}, \qquad s\ge 0.
\end{equation}
In particular, the family of linear maps $(\mathcal{P}_s)_{s\ge 0}$ defined by $\mathcal{P}_s(Y):=H_s Y H_s$ forms the (matrix) heat semigroup associated with the generator $Y\mapsto LY+YL$.
\end{proposition}

\begin{proof}[Proof]
Since $L$ is real symmetric, it is orthogonally diagonalizable: $L = U\Lambda U^\top$ with $U$ orthonormal and $\Lambda=\mathrm{diag}(\lambda_1,\ldots,\lambda_n)$, $\lambda_i\ge 0$. Hence
\[
H_s := e^{-sL}=Ue^{-s\Lambda}U^\top,
\qquad \text{and} \qquad
\frac{\mathrm d}{\mathrm ds}H_s = -LH_s = -H_sL.
\]
Define $Y_s := H_s Y_0 H_s$. Differentiating and using the product rule,
\[
\frac{\mathrm d}{\mathrm ds}Y_s
=
\Big(\frac{\mathrm d}{\mathrm ds}H_s\Big)Y_0H_s + H_sY_0\Big(\frac{\mathrm d}{\mathrm ds}H_s\Big)
=
(-LH_s)Y_0H_s + H_sY_0(-H_sL)
=
-(LY_s + Y_sL),
\]
so $Y_s$ satisfies \eqref{eq:heat-matrix-ode}. Moreover $Y_{0}=H_0Y_0H_0=Y_0$. Uniqueness follows from standard linear ODE theory on $\mathbb{R}^{n\times n}$ (equivalently, by vectorizing the equation to a linear system on $\mathbb{R}^{n^2}$).
\end{proof}

\subsection{Dissipativity and Stability of the Symmetric Heat Diffusion}
\label{app:theory-dissipativity}

In Section~\ref{sec:heat-kernel} we consider the matrix-valued heat equation
\begin{equation}\label{eq:app-heat-matrix-ode}
\frac{\mathrm d}{\mathrm ds}Y_s \;=\; -\big(LY_s + Y_sL\big),\qquad Y_0\in\R^{n\times n},
\end{equation}
where $L\in\R^{n\times n}$ is the (combinatorial) graph Laplacian of an undirected graph, hence symmetric positive semidefinite.

We equip $\R^{n\times n}$ with the Frobenius inner product
\[
\langle A,B\rangle_F := \mathrm{tr}(A^\top B),\qquad \|A\|_F^2 := \langle A,A\rangle_F.
\]
Define the linear operator $\mathcal A:\R^{n\times n}\to\R^{n\times n}$ by
\[
\mathcal A(Y) := LY + YL,
\qquad\text{and}\qquad
\mathcal G(Y) := -\mathcal A(Y) = -(LY+YL).
\]

\begin{lemma}[Dissipativity of the symmetric heat operator]\label{lem:dissipative-LYplusYL}
If $L$ is symmetric positive semidefinite, then $\mathcal G$ is dissipative with respect to $\langle\cdot,\cdot\rangle_F$, i.e.
\[
\langle \mathcal G(Y),\, Y\rangle_F \;\le\; 0
\qquad\text{for all }Y\in\R^{n\times n}.
\]
Equivalently, $\langle \mathcal A(Y),Y\rangle_F\ge 0$ for all $Y$.
\end{lemma}

\begin{proof}
We compute
\[
\langle \mathcal A(Y), Y\rangle_F
= \langle LY, Y\rangle_F + \langle YL, Y\rangle_F
= \mathrm{tr}(Y^\top L Y) \;+\; \mathrm{tr}(Y^\top Y L).
\]
For the first term, write $Y=[y_1,\dots,y_n]$ in columns. Then
\[
\mathrm{tr}(Y^\top L Y) = \sum_{j=1}^n y_j^\top L y_j \;\ge\; 0
\]
since $L\succeq 0$ (i.e. positive semidefinite).

For the second term, note that $Y^\top Y\succeq 0$ and $L\succeq 0$. Using the identity
\[
\mathrm{tr}(AB) = \|A^{1/2}B^{1/2}\|_F^2\quad\text{for }A\succeq 0,\;B\succeq 0,
\]
we obtain
\[
\mathrm{tr}(Y^\top Y L) = \mathrm{tr}\big((Y^\top Y) L\big) \;\ge\; 0.
\]
Hence $\langle \mathcal A(Y), Y\rangle_F\ge 0$ and therefore
$\langle \mathcal G(Y), Y\rangle_F = -\langle \mathcal A(Y), Y\rangle_F\le 0$.
\end{proof}

\begin{corollary}[Energy dissipation and Frobenius stability]\label{cor:frob-stability}
Let $(Y_s)_{s\ge 0}$ solve \eqref{eq:app-heat-matrix-ode}. Then the Frobenius norm is nonincreasing:
\[
\frac{\mathrm d}{\mathrm ds}\|Y_s\|_F^2
= 2\langle \dot Y_s, Y_s\rangle_F
= 2\langle \mathcal G(Y_s), Y_s\rangle_F
\le 0,
\qquad\text{hence}\qquad
\|Y_s\|_F \le \|Y_0\|_F\;\;\forall s\ge 0.
\]
In particular, \eqref{eq:app-heat-matrix-ode} defines a stable linear diffusion on $\R^{n\times n}$ in the sense of contraction in $\|\cdot\|_F$.
\end{corollary}

\begin{proof}
Differentiate $\|Y_s\|_F^2$ and apply Lemma~\ref{lem:dissipative-LYplusYL}.
\end{proof}

\begin{lemma}[Closed-form solution and semigroup contraction]\label{lem:semigroup-contraction}
Let $H_s := e^{-sL}$. The unique solution to \eqref{eq:app-heat-matrix-ode} is
\[
Y_s = H_s\,Y_0\,H_s,\qquad s\ge 0.
\]
Moreover, the maps $\mathcal P_s(Y):=H_s Y H_s$ form a strongly continuous semigroup on $\R^{n\times n}$ and satisfy the contraction bound
\[
\|\mathcal P_s(Y)\|_F \le \|Y\|_F,\qquad s\ge 0.
\]
\end{lemma}

\begin{proof}
The representation $Y_s=H_sY_0H_s$ follows by differentiating $H_sY_0H_s$ and using $\frac{\mathrm d}{\mathrm ds}H_s=-LH_s=-H_sL$.
For contraction, since $L$ is symmetric with eigenvalues $\lambda_i\ge 0$, the eigenvalues of $H_s$ are $e^{-s\lambda_i}\in(0,1]$, hence $\|H_s\|_2\le 1$. 

Recall the submultiplicativity result for the Frobenius norm. For all matrices $A$ and $B$, $\|AB\|_F^2=\text{tr}(B^\top A^\top AB).$ Now since $\|A\|_2^2 I \succeq A^\top A,$ we obtain $\text{tr}(B^\top A^\top AB) \leq \|A\|_2^2\text{tr}(B^\top B)=\|A\|_2^2\|B\|_F^2.$ Taking square roots yields the result $\|AB\|_F \leq \|A\|_2\|B\|_F.$ Using this result, we therefore have,
\[
\|H_s Y H_s\|_F \le \|H_s\|_2\,\|YH_s\|_F \le \|H_s\|_2^2\,\|Y\|_F \le \|Y\|_F. \qedhere
\]
\end{proof}

\subsection{Semigroups and Infinitesimal Generators}
\label{app:semi-generators}

\begin{lemma}[Semigroup structure and infinitesimal generator]\label{lem:semigroup}
Let $\{H_s\}_{s\ge 0}$ be the graph heat kernel $H_s=e^{-sL}$ and define the family of linear operators
\[
P_s : \mathbb{R}^{n\times n}\to\mathbb{R}^{n\times n}, \qquad
P_s(Y) := H_s\, Y\, H_s .
\]
Then $\{P_s\}_{s\ge 0}$ forms a strongly continuous one-parameter semigroup on $\mathbb{R}^{n\times n}$ with $P_0=\mathrm{Id}$ and $P_{s+t}=P_s\circ P_t$. Its infinitesimal generator $\mathcal{A}$ is the linear operator
\[
\mathcal{A}(Y) = LY + YL,
\]
with domain $\mathrm{Dom}(\mathcal{A})=\mathbb{R}^{n\times n}$, and the symmetric heat equation
\[
\frac{\mathrm d}{\mathrm ds}Y_s = -\mathcal{A}(Y_s)
\]
is the abstract evolution equation associated with this semigroup.
\end{lemma}

\begin{proof}
The semigroup property follows directly from $H_{s+t}=H_s H_t$:
\[
P_{s+t}(Y)=H_{s+t}YH_{s+t}=H_s(H_t Y H_t)H_s=(P_s\circ P_t)(Y),
\]
and $P_0(Y)=IYI=Y$. Strong continuity holds since $s\mapsto H_s$ is continuous in operator norm.

For the generator, compute for any $Y\in\mathbb{R}^{n\times n}$,
\[
\lim_{h\downarrow 0}\frac{P_h(Y)-Y}{h}
= \lim_{h\downarrow 0}\frac{H_h Y H_h - Y}{h}
= -LY-YL,
\]
using $H_h=I-hL+o(h)$. Thus $\mathcal{A}(Y)=LY+YL$.
\end{proof}

\subsection{Results for the Forward and Reverse Processes}
\label{app:forward-reverse}

\begin{lemma}
\label{lem:frob-norm}
For a connected graph with combinatorial Laplacian $L$, the Frobenius norm between the symmetric heat diffusion $Y_T = H_T Y_0 H_T$ and its forward process limit $\lim_{s \rightarrow \infty} Y_s = \Pi Y_0 \Pi$ satisfies
\[
    \left\| Y_T - \Pi Y_0 \Pi \right\|_F \le 2e^{-T \lambda_2} \| Y_0 \|_F,
\]
where $\lambda_2$ is the smallest non-zero eigenvalue of $L$.
\end{lemma}
\begin{proof}
Consider the decomposition
\begin{align*}
    \left\| Y_T - \Pi Y_0 \Pi \right\|_F
    &= \left\| H_T Y_0 H_T - \Pi Y_0 \Pi \right\|_F \\
    &\le \left\| H_T Y_0 H_T - \Pi Y_0 H_T \right\|_F + \left\| \Pi Y_0 H_T - \Pi Y_0 \Pi \right\|_F \\
    &= \left\| (H_T - \Pi) Y_0 H_T \right\|_F + \left\| \Pi Y_0 (H_T - \Pi) \right\|_F \\
    &\le \left\| H_T - \Pi \right\|_2 \left\| Y_0 \right\|_F \left\| H_T \right\|_2 + \left\| \Pi \right\|_2 \left\| Y_0 \right\|_F \left\| H_T - \Pi \right\|_2,
\end{align*}
using the property that, for arbitrary matrices, $\left\| ABC \right\|_F \le \left\| A \right\|_2 \left\| B \right\|_F \left\| C \right\|_2$.

$\Pi$ is the rank-1 projection onto constants $\frac{1}{n} \mathbf{1} \mathbf{1}^\top$ with eigenvalue 1, which implies that $\left\| \Pi \right\|_2 = 1$. $H_T = e^{-TL}$ has eigenvalues $e^{-T \lambda_i}$ which implies that $\left\| H_T \right\|_2 = 1$. Since $\mathbf{1}$ is an eigenvector of both $e^{-TL}$ and $\Pi$ with eigenvalue 1, $H_T - \Pi$ can be simultaneously diagonalised. Therefore $\left\| H_T - \Pi \right\|_2 = e^{-T \lambda_2}$ corresponding to the smallest non-zero eigenvalue of $L$. Substituting these equalities into the starting decomposition completes the proof.
\end{proof}

\begin{lemma}
\label{lem:frob-norm-rev}
Let $X_0 = \Pi + X_\perp$ be an initial state for the reverse-time ODE where $X_\perp$ has row sums and column sums equal to 0, namely $\mathbf{1}^\top X_\perp = \mathbf{0}$ and $X_\perp \mathbf{1} = \mathbf{0}$. Reversing the heat diffusion with the true generator $\mathcal{J}X = T(LX + XL)$ produces $X_1 = e^{TL} X_0 e^{TL}$ satisfying
\begin{equation*}
    \left\| X_1 \right\|_F \ge e^{2T \lambda_2} \| X_\perp \|_F.
\end{equation*}
\end{lemma}
\begin{proof}
Observe that
\begin{align*}
    X_1
    &= e^{TL} (\Pi + X_\perp) e^{TL} = \Pi + e^{TL} X_\perp e^{TL},
\end{align*}
since $\mathbf{1}$ is an eigenvector of $e^{TL}$ with eigenvalue 1. Therefore, $\left\| X_1 \right\|_F \ge \left\| e^{TL} X_\perp e^{TL} \right\|_F$.

Let $L = U \Lambda U^\top$ be the eigendecomposition of $L$, meaning $e^{TL} = U e^{T\Lambda} U^\top$. Therefore,
\begin{equation*}
    \left\| e^{TL} X_\perp e^{TL} \right\|^2_F 
    = \left\| e^{T\Lambda} U^\top X_\perp U e^{T\Lambda} \right\|^2_F
    = \sum_{i,j} e^{2T(\lambda_i + \lambda_j)} (U^\top X_\perp U)_{ij}.
\end{equation*}
Since $\mathbf{1}$ is the eigenvalue of $L$ corresponding to $\lambda_1 = 0$, this implies that both $(U^\top X_\perp U)_{i1} = 0$ and $(U^\top X_\perp U)_{1j} = 0$ meaning the $\lambda_1$ terms do not contribute to the sum. For the remaining non-zero terms $\lambda_i + \lambda_j \ge 2 \lambda_2$, therefore
\begin{equation*}
    \left\| e^{TL} X_\perp e^{TL} \right\|^2_F
    \ge e^{4T\lambda_2} \sum_{i,j} (U^\top X_\perp U)_{ij} 
    = e^{4T\lambda_2} \| X_\perp \|^2_F.
\end{equation*}
Taking square roots completes the proof.
\end{proof}

\section{Datasets}
\label{appx:datasets}
For all datasets listed below, when running experiments as described in \cref{sec:behaviour}, or for choosing hyperparameters, we use random seeds in the range $10-20$. For model benchmarking we use seeds $1-5$.

\subsection{Synthetic}
\paragraph{SBM}
We generate random stochastic block model networks (SBMs) with consistent intra-community and inter-community edge probabilities given by $p_{\text{inter}}=0.05$ and $p_{\text{intra}}=0.3$ to allow for reasonable distinguishability of communities. The number of communities is selected at random from $\{2,3,4,5\}$. For general analysis we construct any number of such graphs with arbitrary numbers of nodes using random seeds to control reproducibility. For benchmarking, we retained the seeds $1-5$ and split sets of $N=200$, 200-node graphs into $80\%$ training data and $20\%$ test.

\paragraph{DCSBM}
Degree-corrected SBMs, where the general construction mirrors that of standard SBMs but with additional control over the degree distribution, are generally considered to be more reasonable models for community-structured networks. Our datasets of such graphs are constructed using the same setup as for the SBMs outlined above but with edge probabilities weighted by draws from a $\mathrm{Beta}(1,1)$ distribution.

\paragraph{Planar}
Alongside the stochastic block models described above, we also included planar graphs as an important case study in their own right and as a stark topologically different challenge. In this case we constructed planar graphs by sampling uniformly at random 2D points and connecting this via Delauney triangulation. As above, for general analysis we generated any number of $n$-node planar graphs, retaining specific datasets of $N=200$, $n=64$ with the same random seeds and train/test splits as above.

\subsection{Real-world}
For the purposes of benchmarking on real-world graph datasets, we make use of the popular TUDataset, introduced in \cite{morris2020tudataset}, via the PyTorch Geometric library due to its convenient interface.

\paragraph{Proteins}
Macromolecules represent a rich set of variable-size graphs that include as a subset planar networks. Graphs are constructed by connecting edges between molecular substructures if they neighbour one another along a sequence of amino acids or are one of three nearest-neighbours in space. The dataset was originally introduced as an example problem of identifying enzymes from a set of macromolecules.

\paragraph{Enzymes}
This dataset is much like the proteins dataset, only now all of the elements are enzymes (and the original task was to classify their type).

\paragraph{QM9}
Molecular data constitutes a significant example of real-world planar graphs and generating such graphs is a frequent objective in the community. We follow suit and include the QM9 \citep{wu2018moleculenet} dataset, consisting of small organic molecules ($n\leq9$), along with additional molecular attributes. Following the likes of SPECTRE and DeFoG, we randomly partition the data into training, validation and test sets of size 97,682, 20,026 and 13,123 respectively. At generation time, 512 samples are produced (from each model) to compare to the test set. Test metrics are computed over a maximum of 500 graphs from the test set, to reduce computation time.

\section{Evaluation}
\label{appx:evaluation}
To evaluate the performance of our generator matching model we use synthetic and real-world graph datasets. Having defined a training set of graphs, we follow the procedure described in \cref{sec:training} to train a neural generator and subsequently produce new graphs by sampling from a suitable limiting distribution (corresponding to the choice of generator) and reversing the diffusion process. 

Our goal is to create new sets of graphs whose properties closely resemble those from the training set. Due to the complex structure of networks, defining what those properties should be is in general non-trivial. Substantial effort has been devoted to defining such properties across various mathematical disciplines that make use of such data structures. Within the statistical and machine learning community, we highlight work such as: \cite{wills2020metrics}, in which the authors discuss various commonly used metrics and their relationship to topological features; \cite{thompson2022evaluation}, which introduces metrics specifically intended to aid in measuring the performance of graph-generative models; 
\cite{zheng2023gnnevaluator} which focuses on evaluating the performance of GNNs at node-classification tasks.

Following the approach of precursory generative graph models, we adopt the use of maximum mean discrepancy (MMD) as a way to numerically evaluate the difference between empirical distributions of characteristics. 

The basis of the MMD is in the embedding of probability measures in a reproducing kernel Hilbert space (RKHS). More precisely, given a probability measure $P$, continuous positive-definite real-valued kernel $k$ defined on a (separable) topological space $\mathcal{X}$ with corresponding RKHS $\mathcal{H}$, the kernel mean embedding of $P$ is given by $\mu_P:=\int k(\cdot,x)dP(x)$, provided $k$ is $P$-integrable. With this in mind, we define the MMD between measures $P$ and $Q$ as
\[\MMD(P,Q):=\norm{\mu_P-\mu_Q}_\mathcal{H}.\]
In practice, we usually approximate $\mu_P$ by the Monte Carlo estimate $\mu_{P_n}=\frac{1}{n}\sum_{i=1}^nk(\cdot, X_i)$ with the use of the kernel trick: \[\MMD^2(P,Q)=\E_{X,X'}{k(X,X')}+\E_{Y,Y'}{k(Y,Y')}-2\E_{X,Y}{k(X,Y)}\] where expectations are approximated with Monte Carlo estimates with $X,X'\sim P$ and $Y,Y'\sim Q$. In our case, each contribution to the Monte Carlo estimate corresponds to evaluating the kernel over pairwise empirical distributions of graph properties listed below, each computed from singular graphs.

We leave detailed analysis of this metric (and its estimators) to the likes of \cite{sriperumbudur2010hilbert} and \cite{tolstikhin2016minimax} but suffice to note that generally one should choose a \emph{universal} kernel so that $\mathcal{H}$ is \emph{characteristic} in the sense that $\MMD(P,Q)=0\iff P=Q$; i.e. $\MMD(\cdot,\cdot)$ is a valid metric. 

We use MMD to measure the similarity between generated sets of graphs and hold-out graphs (from the relevant dataset) on the following characteristics:
\begin{description}[labelindent=1cm]
    \item[Degree:] For a node $i$, the degree, $d_i$, denotes the number of edges connected to it.
    \item[Triangles:] $T_i$ represents the number of triangles for which $i$ is a vertex.
    \item[Clustering (coefficient):]  The clustering coefficient for node $i$ is given by $c_i=2T_i/d_i(d_i-1)$.
    \item[Orbit:] Specifically 4-node orbit counts: the number of occurrences of each orbit for graphlets of (upto) size 4.
    \item[Spectrum:] The eigenvalues, $\lambda_i$, of the (normalised) graph Laplacian define the spectrum for the purposes of MMD evaluation.
\end{description}

\section{Implementation Details}
\label{appx:implementation}
\subsection{Neural Network Architecture}
\label{appx:architecture}
Our neural generator is a simple 4-layer ReLU MLP with (in general) 4096 nodes per hidden layer and layer normalisation between layers 2-3 and 3-4. This network acts on a flattened form of the input state and outputs a flattened state that is reshaped into a square matrix. We use the Adam optimiser with an initial learning rate of $10^{-4}$ and a scheduler that reduces this rate by a factor of $0.99$ to a minimum of $10^{-9}$ when the loss stops decreasing. Where we have used data-specific parameter settings, these are given in \cref{tab:hyperparameters}.

[Redacted for anonymity]

\begin{table}[h!]
    \centering
    \begin{tabular}{llllll}
    \toprule
        & Parameter \\
        Dataset & batch size* & epochs & $\alpha$ & $T$ & $w$\\
        \midrule
        Planar & 256 & 200 & $0.1$ & 6 & 4096\\
        SBM & 256 & 20 & 1& 15 & 1024\\
        DCSBM & 8 & 100 & 1& 20 & 4096\\
    \bottomrule
    \end{tabular}
    \caption{Choices of hyperparameters for $G^3$. *: Batch size is the \textit{maximum} batch size; for sets of graphs of different sizes, batches are formed by selecting only graphs of the \textit{same} size, up to a maximum of \texttt{batch\_size}. }
    \label{tab:hyperparameters}
\end{table}

\subsection{Variable numbers of nodes}
When dealing with sets of graphs with different numbers of nodes we can choose to deal with these in different ways. The most na\"{i}ve approach is to simply pad adjacency matrices with zero rows and columns until all of them have $n=\max\{n_1,\dots,n_N\}$ ``nodes''. Using this method, the heat kernel must be similarly padded to ensure matching dimensions.

Clearly, this is both inefficient and undesirable. Whilst padding can be helpful for simplicity, it can be more effective to (additionally) mask these zeros during training by defining suitable modifications to the MLP architecture to account for this. Specifically, we define a \texttt{MaskedLinear} module as a drop in replacement for the standard PyTorch \texttt{Linear} module in our neural architecture. The functionality of this is to take as an argument the number of nodes for a given input and index only the portion of the weight matrix up to the corresponding size of the (flattened) input state (i.e. $\frac{1}{2}n(n-1)$). 

Where datasets include graphs with disconnected nodes, we train only on the largest connected subgraph.

\subsection{Architecture for Related Work}
\label{appx:architecture_related}
Due to hardware limitations, we were unable to run either the DeFoG model exactly as given in the configuration files supplied with it, or the SPECTRE model with the details given in the Readme file. Time constraints led us to reduce the number of epochs in the SPECTRE training from 6000 to 5000. Both time and memory constraints forced us to make the following changes to the DeFoG SBM configuration file (which was duplicated in the case of DCSBMs, there being no pre-specified configuration file in that instance):

\begin{table}[h]
    \centering
    \begin{tabular}{lll}
    \toprule
        & Original & Modified \\
        Parameter &  \\
        \midrule
        batch size & 32 & 8 \\
        epochs & 50000 & 40000 \\
        layers & 8 & 4 \\
        RRWP steps & 20 & 10 \\
    \bottomrule
    \end{tabular}
    \caption{Table of modified parameter choices for the DeFoG SBM config. file.}
    \label{tab:defog_sbm_config}
\end{table}

The configuration file for planar graphs was left unchanged for DeFoG. We noted that, in spite of different training parameters being provided in the SPECTRE Readme file for planar graphs (vs SBMs), we obtained better performance with the SBM settings than we did with the planar settings. For the Enzymes dataset, for which neither SPECTRE nor DeFoG has recorded configurations, we adopted very similar parameter settings to those provided for the Proteins dataset, due to the relative similarity between the graphs in these two datasets. The only differences are as follows: for SPECTRE, we reduced the maximum number of training epochs from 1020 to 400 due to time constraints; for DeFoG, although dataset infrastructure was inplace for Proteins, there was no configuration file. As such, we adopted the planar settings but with the epochs reduced from 100000 to 1000 for Proteins and 2000 for Enzymes.

\paragraph{DeFoG* and SPECTRE*}
In \cref{sec:performance} we introduced DeFoG* and SPECTRE* as computationally limited forms of each model. \cref{tab:defog*_spectre*} shows how we constrained relevant parameters to reduce computation time.

\begin{table}[h]
    \centering
    \begin{tabular}{llll}
    \toprule
          &  & Original & Modified \\
        Model & Data &  \\
        \midrule
        SPECTRE* & Proteins & 1020 & 10 \\
        SPECTRE* & Enzymes & 400 & 4 \\
        DeFoG* & SBM & 40000 & 1000 \\
        DeFoG* & DCSBM & 40000 & 1000 \\
        DeFoG* & Planar & 100000 & 1000 \\
    \bottomrule
    \end{tabular}
    \caption{Table of training epoch reductions for DeFoG* and SPECTRE* by dataset.}
    \label{tab:defog*_spectre*}
\end{table}

\section{Framework Extensions}
\label{appx:extensions}

\subsection{Asymmetric Heat Kernel}
\label{appx:asymmetric}
In \cref{sec:background} we consider the heat equation on a graph and define a \textit{symmetric heat diffusion} that applies the heat kernel to the left and right of some graph matrix representation. We go on to outline the behaviour of this process and how it fits into the generator matching methodology. 

In this section, we consider the more obvious \textit{asymmetric} heat diffusion, where $Y_s:=H_sY_0$ (using the same notation as in \cref{sec:background}) and provide an abridged summary of the corresponding quantities and modifications to $G^3$, primarily through \cref{tab:symm_v_asymm}. Many of the asymmetric quantities are intuitive, however the choice of $X_0$, which depends on the choice of matrix representation $Y_0$, requires further explanation. For example, in the case of $Y_0=A$, the standard adjacency matrix, we have a limit matrix of $\frac{1}{n}\1\bm{d}^\top$, where $\bm{d}$ is the vector of degrees of $A$. From this, we could define a base distribution with elements $X_0:=\bm{\tilde{x}}_0\bm{v}_0^\top$, where $\bm{\tilde{x}}_0\sim\mathrm{Dirichlet}(\alpha\1)$ and $\log\bm{v}_0\sim\mc{N}(\hat{\mu}_d,\hat{\sigma}_d^2)$ with $\hat{\mu}_d,\hat{\sigma}_d^2$ the logarithm of the mean and variance of the scaled degree vectors from the training set, i.e. $\hat{\mu}_d:=\log(n^{-1}\avg{d})$. On the other hand, choosing $Y_0=D^{-1}A$ (the random walk matrix) leads to a limit matrix of $n^{-1}\1\1^\top$ and the natural choice of a pure Dirichlet base distribution, much like that described in the main text only without the symmetrisation and scaling.

\begin{table}[h]
    \centering
    \begin{tabular}{ccc}
    \toprule
      Quantity   & Symmetric & Asymmetric\\
    \midrule
        $\mc{P}_s(Y)$ & $H_sY H_s$ & $H_sY$\\
        $\ddt Y_t$ & $-LY_t-Y_tL$ & $-LY_t$ \\
        $\lim_{t\rightarrow\infty}Y_t$ & $\tfrac1n(\1^\top Y_0\1)\1\1^\top$ & $\tfrac1n\1\1^\top Y_0$\\
        $X_0(A)$ & $\avg{d}(\tilde{X}_0 + \tilde{X}_0^\top)$ & $\bm{\tilde{x}}_0\bm{v}_0^\top$ \\
    \bottomrule
    \end{tabular}
    \caption{Definitions of key quantities under the symmetric and asymmetric heat kernel processes.}
    \label{tab:symm_v_asymm}
\end{table}

\subsection{Alternate Graph Representations and Heat Kernels}
\label{appx:method}

In \cref{sec:training}--\ref{sec:sampling} we describe the general scheme for heat kernel generator matching and sampling. To build a practical implementation of this scheme the user still has various choices to make. The most salient of these is the choice of matrix-representation of the graph, as this dictates a corresponding base distribution. As described in the preceding section, under symmetric and asymmetric formulations, a choice of $Y_0=A$ leads to different samples $X_0$. Other representations are possible, such as the random walk matrix $D^{-1}A$ or the symmetrised form $D^{-\tfrac12}AD^{-\tfrac12}$.


For simplicity, the main text focusses on the heat kernel formulation using the combinatorial Laplacian. However, it is straightforward to generalise the approach to other kernels, provided they satisfy suitable stochasticity properties. One alternative is the \textit{random walk kernel}, which is shown to be a valid choice in \cref{lem:rw_validity}. If such an alternative is used, suitable limiting matrices and base distributions must be computed. In \cref{tab:kernel_val_lim} we consider replacing the combinatorial Laplacian with the random walk and normalised Laplacians, $L_\text{rw}=I-D^{-1}A$ and $L_N=I-D^{-\tfrac12}AD^{-\tfrac12}$ as well as the random walk and lazy random walk kernels. 

\begin{lemma}
\label{lem:rw_validity}
    The (lazy) random walk kernel is a valid probability transition kernel.
\end{lemma}
\begin{proof}
    Let $Q\Lambda Q^{-1}$ be the eigendecomposition of $D^{-1}A$. We have that $D^{-1}A\1=\1$, hence $\1$ is an eigenvector with eigenvalue 1. By eigenvector orthogonality we know that $Q^{-1}\1=c^{-1}\bm{e}_1$. Thus, $Q\Lambda^t Q^{-1}\1=Q\Lambda^tc^{-1}\bm{e}_1=Q\lambda_1^tc^{-1}\bm{e}_1=\bm{q}_1/c=\1$ since $\lambda_1=1$ and we have ordered the eigenvalues such that $\bm{q}_1=Q_{:1}=c\1$ is the first column vector of $Q$. The same argument applies in the case of the lazy random walk operator, $\half(I+D^{-1}A)$.
\end{proof}

The limiting form of the random walk matrix is given by $(D^{-1}A)^t\xrightarrow{t\rightarrow\infty}QJ^{11}Q^{-1}$, where $J^{ij}$ is the matrix of zeros with a single 1 at the $ij$th entry. This simplifies to the rank-1 matrix $\bm{q}_1\bm{q}^{-\top}_1$, where $\bm{q}_1^{-\top}$ is the 1st row of $Q^{-1}$ which further simplifies to $c\1\bm{q}_1^{-\top}$, since the first eigenvector is a constant multiple of the vector of ones. The same argument applies to the lazy random walk matrix. In practice this quantity tends to be close to the heat limit anyway; the perturbation comes from the lack of symmetry of the random walk matrix, which can be quantified by the commutator $[D^{-1},A]=D^{-1}A-AD^{-1}$.

\begin{table}[h]
    \centering
    \begin{tabular}{ccc}
    \toprule
        Kernel & Valid & Limit\\
    \midrule
        $\e^{-tL}$ & \checkmark & $n^{-1}\1\1^\top$\\
        $\e^{-tL_\text{rw}}$ & \checkmark & $n^{-1}\1\1^\top$ \\
        $\e^{-tL_N}$ & \xmark  & $n^{-1}\1\1^\top$\\
        $(D^{-1}A)^t$ & \checkmark & $q_{11}\1\bm{q}_1^{-\top}$\\
        $\frac{1}{2^t}(I+D^{-1}A)^t$ & \checkmark & $q_{11}\1\bm{q}_1^{-\top}$\\
    \bottomrule
    \end{tabular}
    \caption{Validity of probabilistic interpretation and limiting behaviour for different kernels.}
    \label{tab:kernel_val_lim}
\end{table}

Further, any such kernel can be symmetrised to act on both the right and left of the subject matrix as is described in the main text for the heat kernel.

\subsection{Conditional Graph Generation}
\label{appx:conditional_generation}

Another choice of matrix representation leads to a bonus formulation that provides a natural way to include covariate information for the purposes of conditional generation (often known as ``guidance'' in the generative literature). Covariate-assisted spectral clustering \cite{binkiewicz2017covariate} uses node-level covariate information to aid spectral-based clustering techniques. The key to the approach is to augment the Laplacian in the following way: $L_{\omega}:= L+\omega ZZ^\top$ for some tuning parameter $\omega$ and covariates $Z$. We adopt this augmentation for our matrix representation ($Y_0$) and observe that, for vector-valued $Z$ (denoted $\bm{z}$), $\lim_{t\rightarrow\infty}H_tL_{\omega}=n^{-1}\omega\1(\1^\top\bm{z})\bm{z}^\top$; in other words, the limiting distribution depends \textit{explicitly} on the node covariates. This leads to a natural way to implement graph generation conditioned on pre-specified covariate information by constructing suitable limit matrices analogously to those described in \cref{appx:method}.

In \cref{fig:sbm_guidance_example}, we demonstrate evidence of this approach in action on a toy example of SBM generation. In each column of the figure we generate different numbers of communities, with pre-chosen membership ratios: this information is encoded directly into vectors $(\bm{z}_i)_{i=1}^{N^*}$ for $N^*$ new graphs and used to construct matrices of the form $X_0=n^{-1}\hat{\omega}(\bm{\tilde{x}}_0^\top\bm{z})\bm{\tilde{x}}_0\bm{z}^\top$ for $\bm{\tilde{x}}_0\sim\mathrm{Dirichlet}(\alpha\1)$.

\begin{figure}
    \centering
    \includegraphics[scale=0.6]{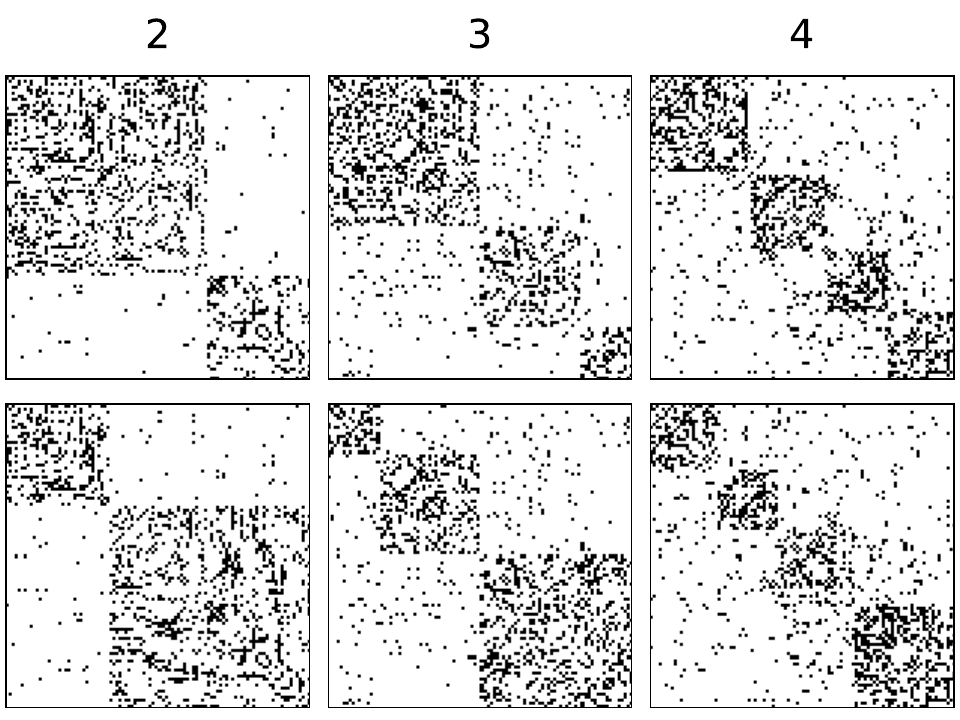}
    \caption{SBM adjacency matrices generated by pre-specifying node labels to the conditional form of the $G^3$ model. Columns headers indicate the number of communities and the top row uses a given ratio of membership sizes with the bottom row using the sizes reversed in order.}
    \label{fig:sbm_guidance_example}
\end{figure}

\section{Additional Experimental Results}
\label{appx:experiments}
In this section we include figures indicating the behaviour of $G^3$ with respect to training (and sampling) variables, including the number of graphs, number of nodes and choice of $\alpha$ parameter for the base distribution. 

\cref{fig:spec_v_tmax_by_n_planar} shows spectral MMD as a function of $T$ and the number of nodes in the graphs. Perhaps surprisingly, the apparent minima of the curves seems to occur at approximately the same $T$ which perhaps indicates that $N$ is the more important factor. The curves also appear to suggest that the problem gets ``harder'' as $n$ increases, at least for the planar graphs shown here.

\cref{fig:non-unique_v_tmax_by_N_planar} supplements \cref{fig:spec_v_tmax_by_N_planar} and \cref{fig:spec_v_tmax_by_n_planar} by showing how the proportion of non-unique output graphs depends on $n$ and $N$ (and $T$). Roughly, the spectral performance and non-uniqueness have some negative correlation, indicating that care needs to be taken to optimise overall quality of graph output.

\begin{table*}\footnotesize
\centering
\caption{Median (± median abs. deviation) of MMD metrics for $G^3$, DeFoG*, SPECTRE* over $3-5$ randomly seeded splits. Lower is better and \textbf{bold} indicates best performance, based on the median plus one median abs.~deviation. Training-test set comparisons are also included to show best possible performance.}
\label{tab:short_models_train_mmd}
\begin{tabularx}{0.95\linewidth}{llccccc}
\toprule
Dataset & Model & Clustering & Degree & Orbit & Spectrum & Triangles \\
\midrule
\multirow[c]{3}{*}{SBM}
    & $G^3$        & \bfseries 0.0356 ± 0.0005 & \bfseries 0.0318 ± 0.0052 & 0.0631 ± 0.0120 & 0.0796 ± 0.0120 & \bfseries 0.0270 ± 0.0078 \\
    & Asymm.~$G^3$ & 0.0398 ± 0.0010 & 0.0390 ± 0.0052 & 0.0710 ± 0.0140 & \bfseries 0.0832 ± 0.0015 & 0.0287 ± 0.0084 \\
    & DeFoG*       & 0.0513 ± 0.0005 & 0.6270 ± 0.0027 & \bfseries 0.0500 ± 0.0000 & 1.6600 ± 0.0003 & 0.0818 ± 0.0030 \\
    & Training     & 0.0312 ± 0.0000 & 0.0028 ± 0.0013 & 0.0312 ± 0.0000 & 0.0055 ± 0.0004 & 0.0017 ± 0.0002 \\
\midrule
\multirow[c]{3}{*}{DCSBM}
    & $G^3$        & 0.1270 ± 0.0091 & \bfseries 0.0344 ± 0.0062 & 0.0971 ± 0.0057 & 0.1230 ± 0.0330 & \bfseries 0.0110 ± 0.0098 \\
    & Asymm.~$G^3$ & 0.2340 ± 0.0160 & 0.0810 ± 0.0200 & 0.1420 ± 0.0170 & \bfseries 0.1010 ± 0.0220 & 0.0578 ± 0.0250 \\
    & DeFoG*       & \bfseries 0.1110 ± 0.0056 & 0.7370 ± 0.0015 & \bfseries 0.0665 ± 0.0003 & 1.2600 ± 0.0041 & 0.4330 ± 0.0210 \\
    & Training     & 0.0291 ± 0.0009 & 0.0017 ± 0.0007 & 0.0286 ± 0.0025 & 0.0277 ± 0.0069 & 0.0012 ± 0.0002 \\
\midrule
\multirow[c]{4}{*}{Planar}
    & $G^3$        & 0.3090 ± 0.0300 & \bfseries 0.0030 ± 0.0015 & \bfseries 0.0035 ± 0.0018 & \bfseries 0.0060 ± 0.0015 & 0.0467 ± 0.0360 \\
    & Asymm.~$G^3$ & 1.0200 ± 0.0440 & 0.0074 ± 0.0050 & 0.0106 ± 0.0089 & 0.0113 ± 0.0021 & 0.1110 ± 0.0320 \\
    & DeFoG*       & \bfseries 0.1200 ± 0.0043 & 0.7350 ± 0.0010 & 1.0100 ± 0.0014 & 1.9500 ± 0.0160 & \bfseries 0.0296 ± 0.0005 \\
    & Training     & 0.0224 ± 0.0021 & 0.0004 ± 0.0000 & 0.0002 ± 0.0001 & 0.0004 ± 0.0002 & 0.0050 ± 0.0016 \\
\midrule
\multirow[c]{3}{*}{Enzymes}
    & $G^3$        & \bfseries 0.0225 ± 0.0008 &  \bfseries 0.0854 ± 0.0550 & 0.0902 ± 0.0390 & \bfseries 0.0755 ± 0.0220 & \bfseries 0.0490 ± 0.0380 \\
    & Asymm.~$G^3$ & 0.0803 ± 0.0200 & 0.0934 ± 0.0520 & \bfseries 0.1040 ± 0.0047 & 0.1090 ± 0.0480 & 0.0752 ± 0.0120 \\
    & SPECTRE*     & 0.0598 ± 0.0110 & 0.7290 ± 0.0043 & 0.6140 ± 0.0230 & 1.1700 ± 0.0590 & 0.4700 ± 0.0096 \\
    & Training     & 0.0122 ± 0.0002 & 0.0004 ± 0.0004 & 0.0024 ± 0.0014 & 0.0114 ± 0.0024 & 0.0021 ± 0.0001 \\
\midrule
\multirow[c]{3}{*}{Proteins}
    & $G^3$        & \bfseries 0.0465 ± 0.0009 & \bfseries 0.0220 ± 0.0012 &  0.5530 ± 0.0720 & 0.1600 ± 0.0074 &  0.0928 ± 0.0180 \\
    & Asymm.~$G^3$ &  0.0750 ± 0.0073 &  0.0345 ± 0.0003 & \bfseries 0.3480 ± 0.2000 & \bfseries 0.0921 ± 0.0110 & \bfseries 0.0419 ± 0.0071 \\
    & SPECTRE*     & 0.1070 ± 0.0000 & 0.3600 ± 0.0650 & 0.7820 ± 0.0250 & 0.3750 ± 0.0044 & 0.0358 ± 0.0240 \\
    & Training     & 0.0504 ± 0.0001 & 0.0012 ± 0.0001 & 0.0093 ± 0.0010 & 0.0568 ± 0.0003 & 0.0011 ± 0.0003 \\
\bottomrule
\end{tabularx}
\end{table*}

\begin{figure}
    \centering
    \begin{subfigure}[h]{\FigureScale\textwidth}
        \includegraphics[width=\linewidth]{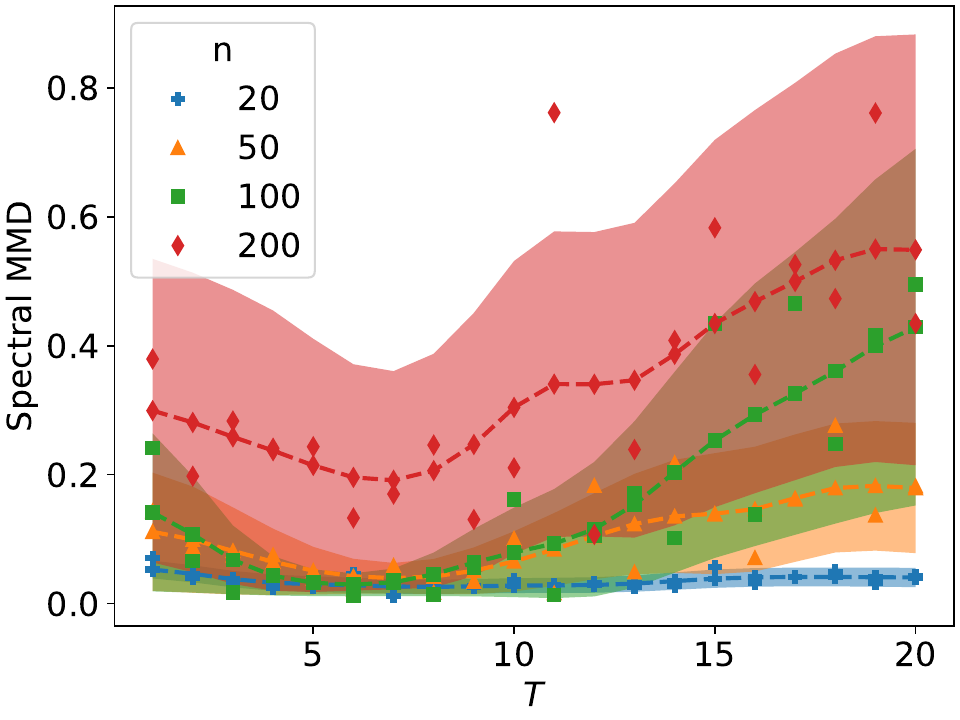}
        \caption{}\label{fig:spec_v_tmax_by_n_planar}
    \end{subfigure}
    \begin{subfigure}[h]{\FigureScale\textwidth}
        \includegraphics[width=\linewidth]{figures/spec_v_tmax_by_n_sbm.pdf}
        \caption{}\label{fig:spec-n-sbm}
    \end{subfigure}
    \caption{Spectral MMD as a function of $T$ and number of nodes, $n$, for ($N=100$) \subref{fig:spec_v_tmax_by_n_planar} planar graphs and \subref{fig:spec-n-sbm} SBMs. \SmoothRibbonText }
    \label{fig:spec_v_tmax_by_n}
\end{figure}

\begin{figure}
    \centering
    \begin{subfigure}[h]{0.48\textwidth}
            \centering
            \begin{tabular}{llll}
            \toprule
             &  &  & spec \\
            $N$ & $T$ & $w$ &   \\
            \midrule
            \multirow[c]{3}{*}{100} & \multirow[c]{3}{*}{15} & 1024 & 0.020 $\pm$ 0.019 \\
             &  & 2048 & 0.027 $\pm$ 0.022 \\
             &  & 4096 & 0.041 $\pm$ 0.043 \\
            \cline{1-4} \cline{2-4}
            \bottomrule
            \end{tabular}
           \caption{}\label{tab:sbm_w}
    \end{subfigure}
    \begin{subfigure}[h]{0.48\textwidth}
        \includegraphics[scale=\FigureScale]{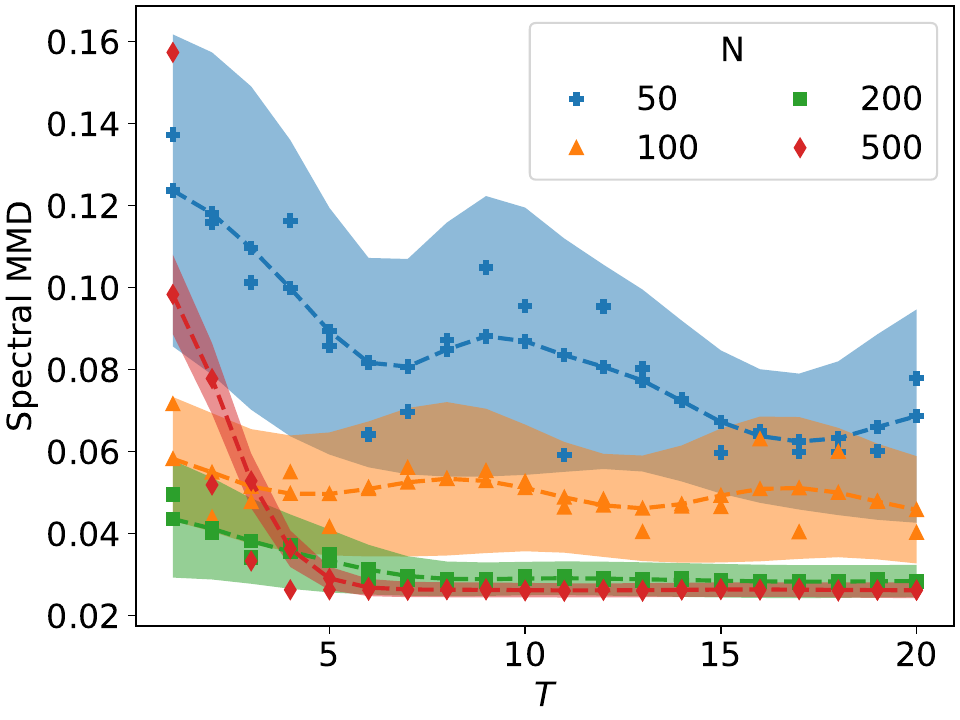}
            \caption{}\label{fig:spec-N-sbm}
    \end{subfigure}
    \caption{\subref{tab:sbm_w} Spectral MMD for different choices of $w$, the number of hidden nodes per layer, for fixed $T$ and $N$ (for SBMs). \subref{fig:spec-N-sbm}: Spectral MMD as a function of $T$ and number of graphs $N$ for SBMs. \SmoothRibbonText }
    \label{fig:spec_v_tmax_by_N_sbm}
\end{figure}

\begin{figure}
    \centering
    \begin{subfigure}[h]{\FigureScale\textwidth}
        \includegraphics[width=\linewidth]{figures/non-unique_v_tmax_by_N_planar.pdf}
        \caption{}\label{fig:non-unique-N}
    \end{subfigure}
    \begin{subfigure}[h]{\FigureScale\textwidth}
        \includegraphics[width=\linewidth]{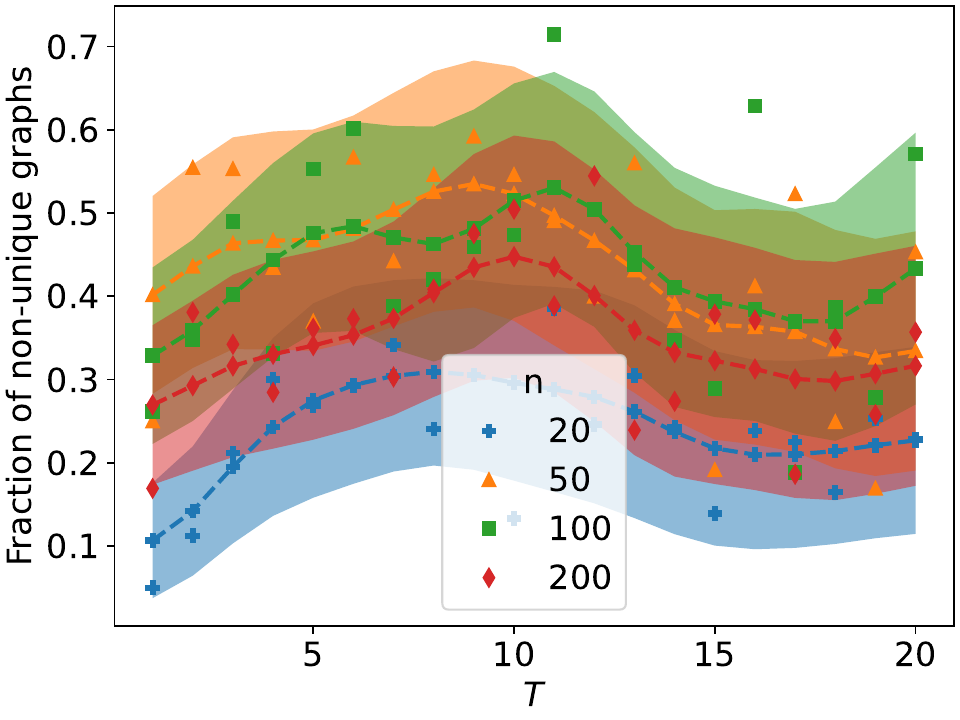}
        \caption{}\label{fig:non-unique-n}
    \end{subfigure}
    \caption{Proportion of non-unique generated graphs, as a function of $T$. \subref{fig:non-unique-N}: Planar graphs ($n=64$) for different sizes of dataset, $N$. \subref{fig:non-unique-n}: Planar graphs ($N=100$) for different sized graphs. \SmoothRibbonText }
    \label{fig:non-unique_v_tmax_by_N_planar}
\end{figure}

\begin{figure}
    \centering
    \includegraphics[scale=\FigureScale]{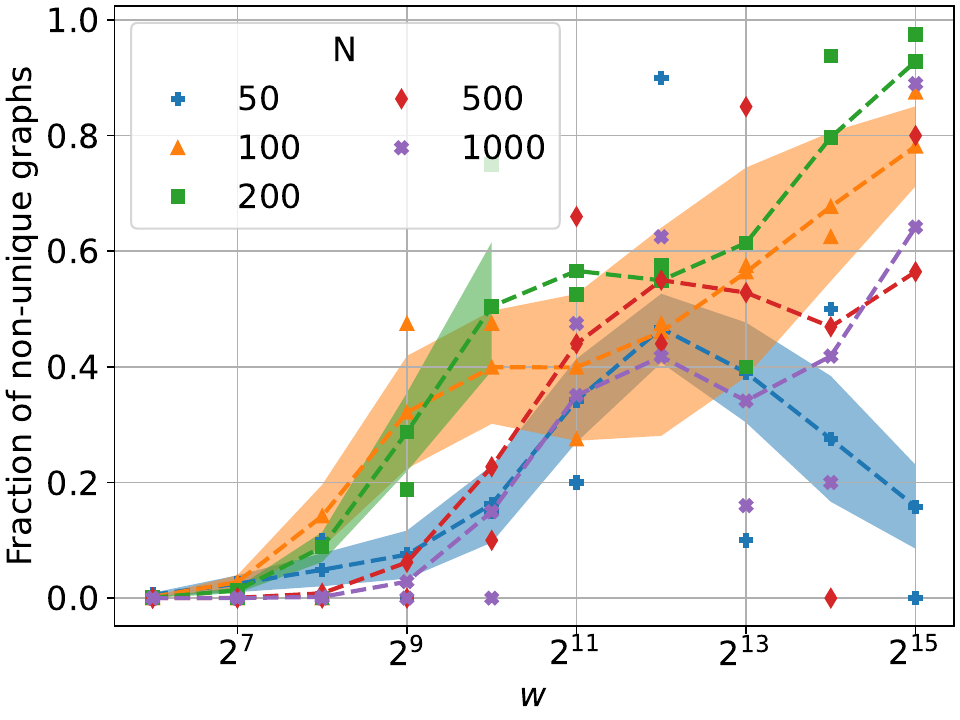}
    \caption{Proportion of non-unique generated planar graphs as a function of $w$, the number of nodes per hidden layer, and number of graphs $N$. \SmoothRibbonText}
    \label{fig:non-unique_v_w_by_N_planar}
\end{figure}

\begin{figure}
    \centering
    \begin{subfigure}[h]{\FigureScale\textwidth}
        \includegraphics[width=\linewidth]{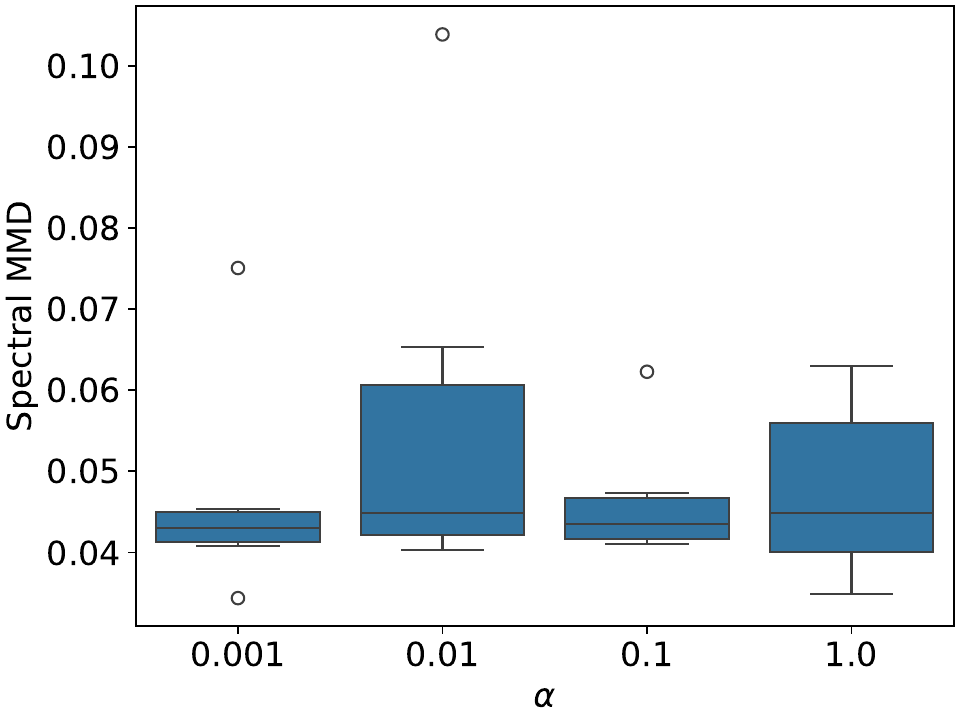}
        \caption{}\label{fig:alpha_sbm}
    \end{subfigure}
        \begin{subfigure}[h]{\FigureScale\textwidth}
        \includegraphics[width=\linewidth]{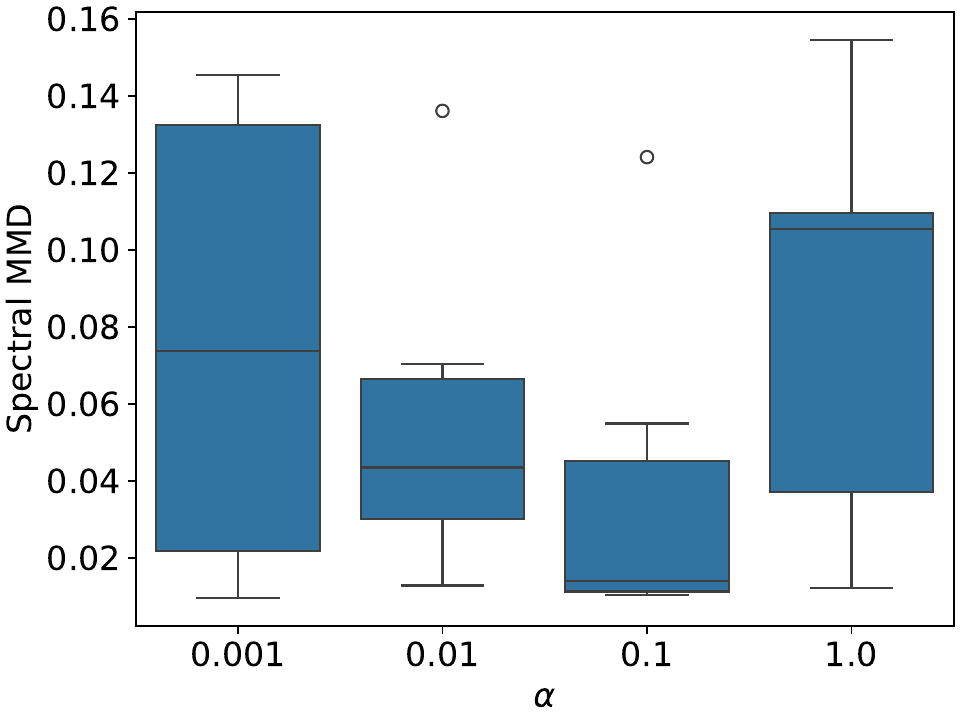}
        \caption{}\label{fig:alpha_planar}
    \end{subfigure}
     \caption{Spectral MMD as a function of $\alpha$, the Dirichlet concentration parameter for the base distribution. \subref{fig:alpha_sbm}: results for 200-node SBM graphs; \subref{fig:alpha_planar}: results for 64-node planar graphs.}
    \label{fig:spec_v_alpha}
\end{figure}

\begin{figure}
    \centering
     \includegraphics[scale=\FigureScale]{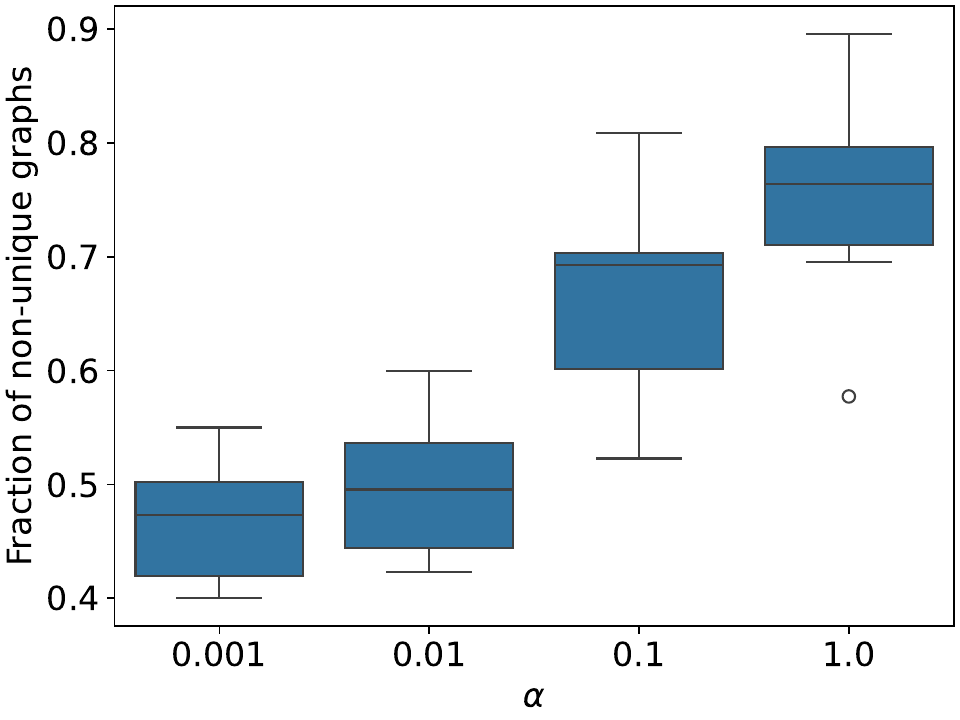}
     \caption{Fraction of (planar) graphs generated that are non-unique as a function of $\alpha$.}
    \label{fig:nonu_v_alpha}
\end{figure}



\begin{figure}
    \centering
    \includegraphics[scale=\FigureScale]{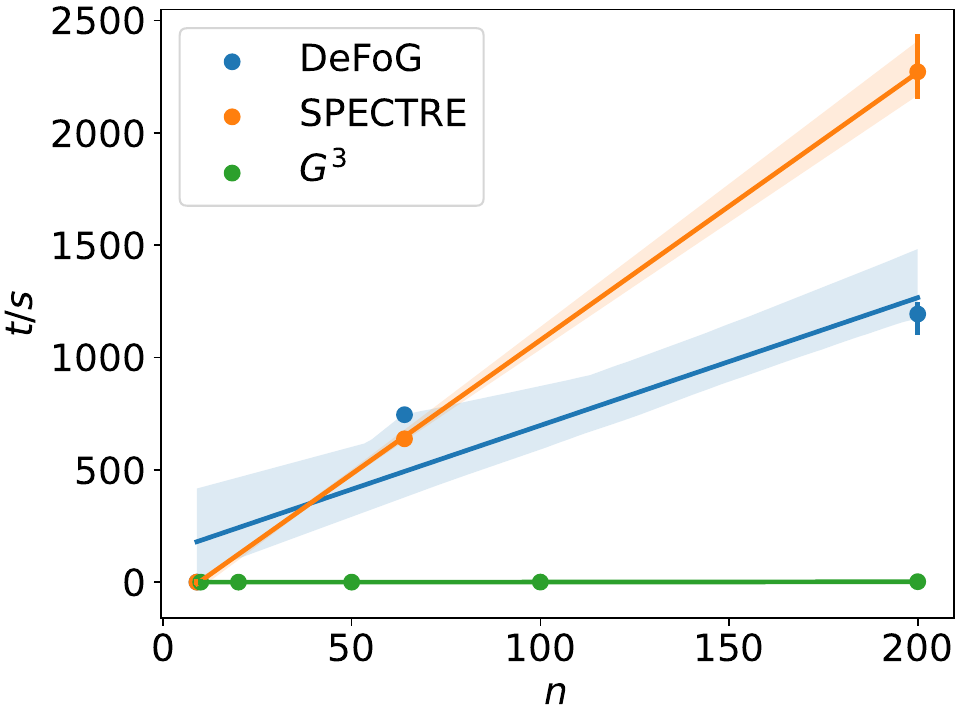}
    \caption{Training plus generation time per sample graph for DeFoG, SPECTRE, and $G^3$.}
    \label{fig:train_times_v_n}
\end{figure}

\def\ExampleFigScale{0.49}
\begin{figure}
    \centering
        \begin{subfigure}[h]{\ExampleFigScale\textwidth}
        \includegraphics[width=\linewidth]{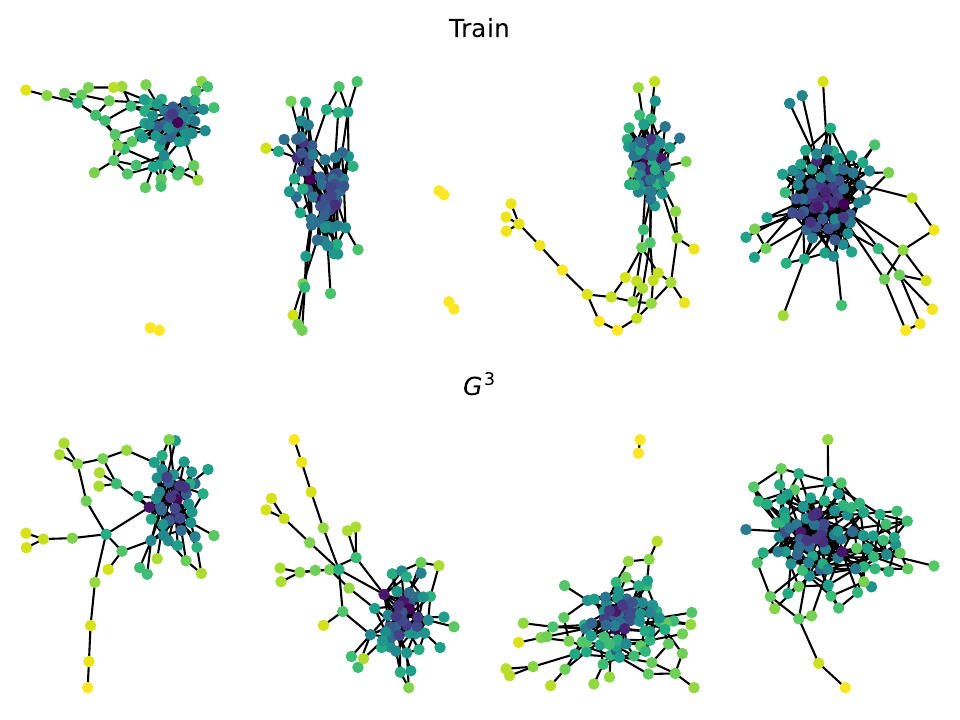}
        \caption{}\label{fig:example_dcsbm}
    \end{subfigure}
        \begin{subfigure}[h]{\ExampleFigScale\textwidth}
        \includegraphics[width=\linewidth]{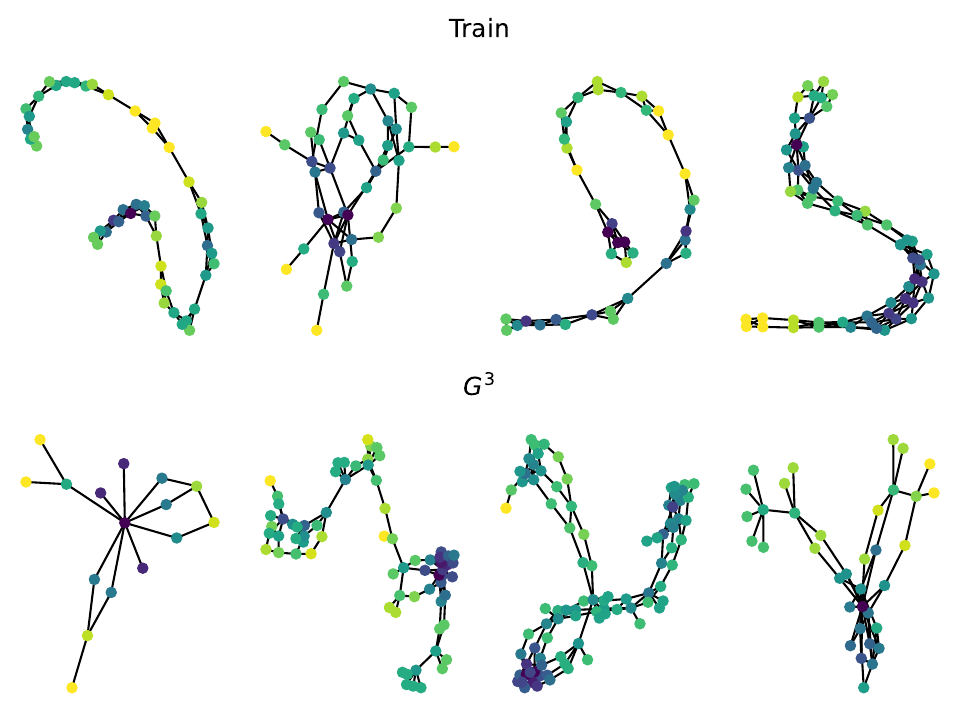}
        \caption{}\label{fig:example_proteins}
    \end{subfigure}
        \begin{subfigure}[h]{\ExampleFigScale\textwidth}
        \includegraphics[width=\linewidth]{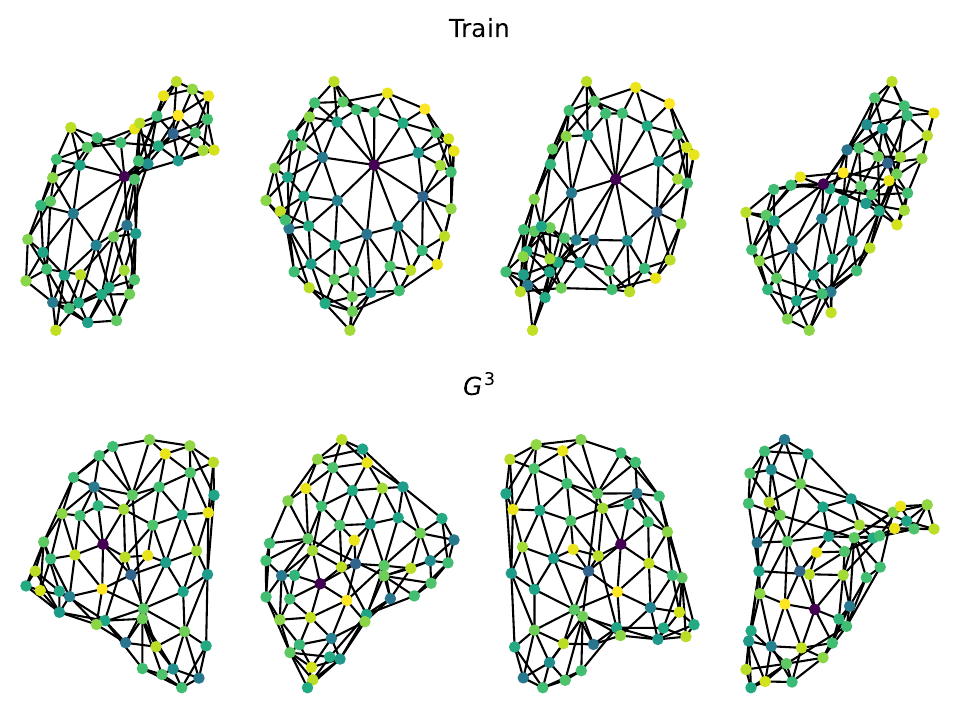}
        \caption{}\label{fig:example_planar}
    \end{subfigure}
        \begin{subfigure}[h]{\ExampleFigScale\textwidth}
        \includegraphics[width=\linewidth]{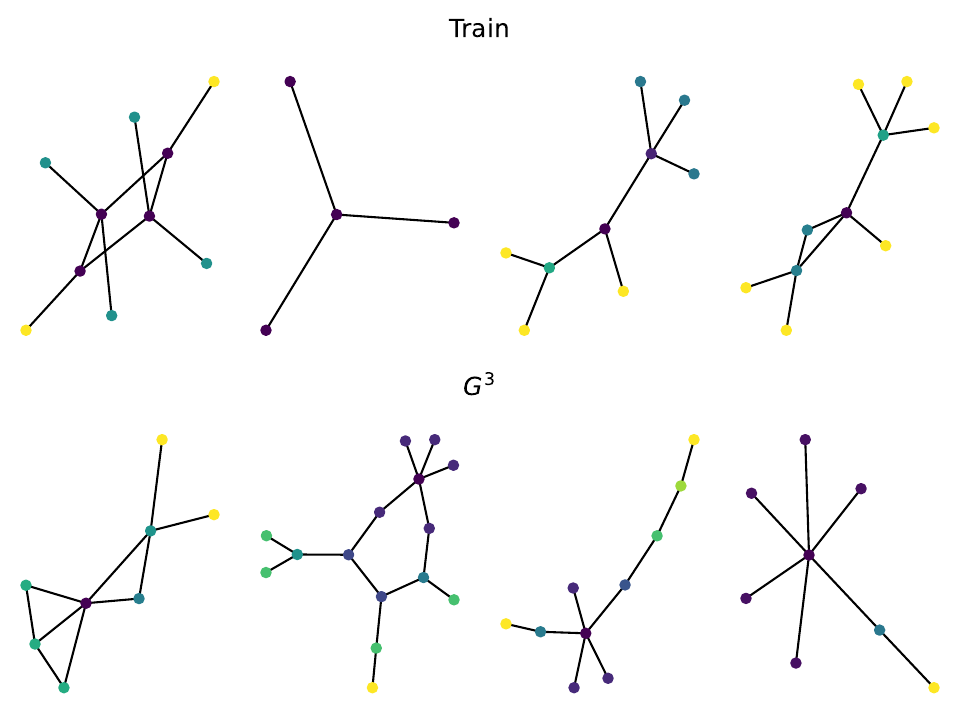}
        \caption{}\label{fig:example_qm9}
    \end{subfigure}
     \caption{A series of example graphs drawn at random from a training set (top) and the output of $G^3$ (bottom). \subref{fig:example_dcsbm}: DCSBM, \subref{fig:example_proteins}: Protein, \subref{fig:example_planar}: Planar, \subref{fig:example_qm9}: QM9. Node colouring reflects Forman-Ricci curvature at that node.}
    \label{fig:example_graphs}
\end{figure}


\begin{figure*}[t]
    \centering
     \includegraphics[scale=0.5]{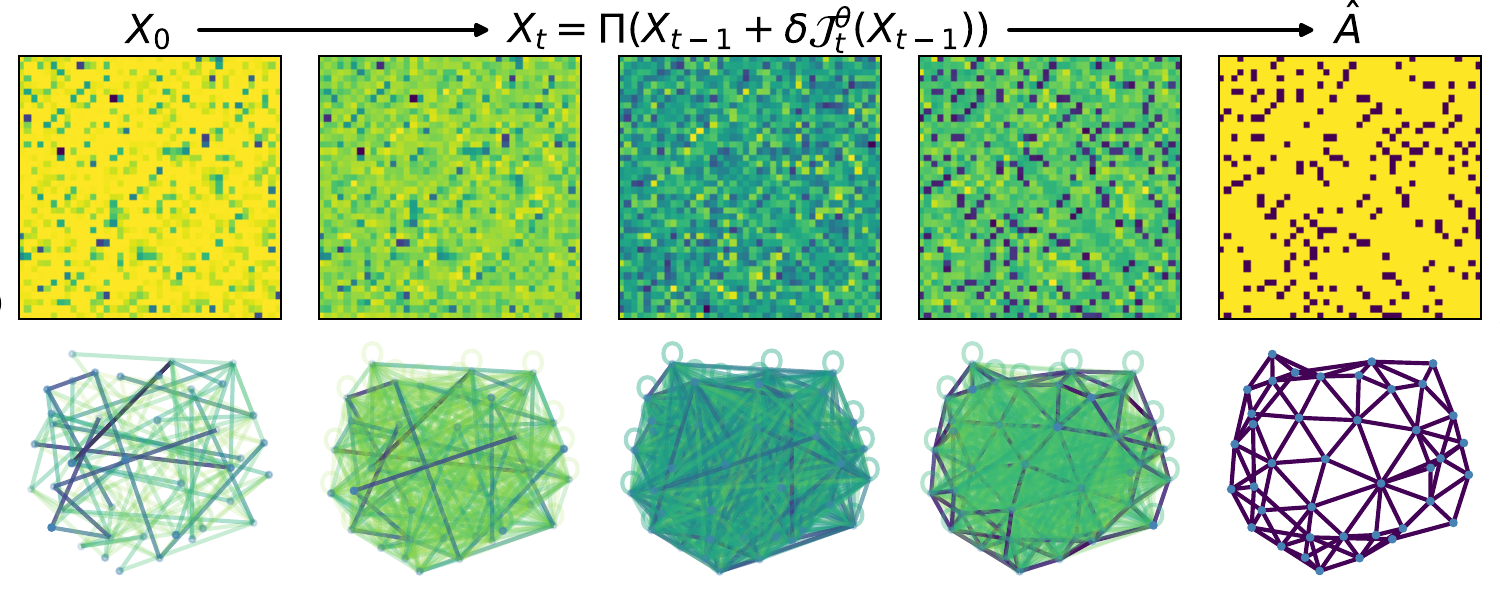}
     \caption{Example reconstruction of a graph from a sample from the base distribution.}
    \label{fig:reverse_process}
\end{figure*}


\end{document}